\documentclass{article}

\PassOptionsToPackage{numbers, compress}{natbib}

\usepackage[preprint]{neurips_2020}

\usepackage[utf8]{inputenc} %
\usepackage[T1]{fontenc}    %
\usepackage{hyperref}       %
\usepackage{url}            %
\usepackage{booktabs}       %
\usepackage{amsfonts}       %
\usepackage{nicefrac}       %
\usepackage{microtype}      %

\usepackage{graphicx}
\usepackage{subfigure}
\usepackage{wrapfig}
\usepackage{framed}
\usepackage{listings}
\usepackage[inline]{enumitem}

\usepackage{algorithm}
\usepackage{algpseudocode}
\usepackage{multirow}
\usepackage[table]{xcolor}

\usepackage{minitoc}

\usepackage{symbols}

\title{Strong Generalization and Efficiency \\in Neural Programs}

\author{
  Yujia Li,
  Felix Gimeno,
  Pushmeet Kohli,
  Oriol Vinyals\\
  DeepMind\\
  \texttt{\{yujiali,fgimeno,pushmeet,vinyals\}@google.com}
}

\begin{document}

\maketitle

\begin{abstract}
We study the problem of learning efficient algorithms that strongly generalize in the framework of neural program induction. By carefully designing the input / output interfaces of the neural model and through imitation, we are able to learn models that produce correct results for arbitrary input sizes, achieving strong generalization. Moreover, by using reinforcement learning, we optimize for program efficiency metrics, and discover new algorithms that surpass the teacher used in imitation. With this, our approach can learn to outperform custom-written solutions for a variety of problems, as we tested it on sorting, searching in ordered lists and the NP-complete 0/1 knapsack problem, which sets a notable milestone in the field of Neural Program Induction.
As highlights, our learned model can perform sorting perfectly on any input data size we tested on, with $O(n\log n)$ complexity, whilst outperforming hand-coded algorithms, including quick sort, in number of operations even for list sizes far beyond those seen during training.
\end{abstract}

\doparttoc
\faketableofcontents

\section{Introduction}

Innovation in algorithm design has two important goals: \emph{correctness} and \emph{efficiency}.  \emph{Correctness} is about producing the right outputs for arbitrary inputs, and \emph{efficiency} is about consuming as few resources as possible such that an algorithm runs fast, or uses little memory.

Neural networks have been studied as a way to represent algorithms for solving concrete problems such as sorting numbers or finding information in databases \cite{graves2014neural,reed2015neural,vinyals2015order,neelakantan2015neural}, but with two important limitations. First,   
few research efforts have been able to demonstrate strong generalization, \ie{}~correctness beyond the training distribution. Second, efficiency has been hard to demonstrate, or has not been the main focus of such efforts, although improving efficiency is one of the most appealing consequences of learning algorithms.

The objective of this work is to learn algorithms represented by neural networks that strongly generalize and are more efficient than hand-crafted algorithms. For the first time, we show that this objective can be attained, via a combination of imitation and RL, and a judicious design of the input and output interfaces of the neural networks.

The framework we use throughout our study is that of neural program induction \cite{graves2014neural,joulin2015inferring,reed2015neural,gaunt2016terpret,devlin2017neural}, which can be stated as having a neural controller $f$ which, from a particular program execution state $s$, issues an instruction $a = f(s)$ which has some pre-determined semantics over how it transforms $s$. This is iterated up to meeting some termination condition or a maximum number of steps.  The neural network parameterizes a policy distribution $p(a|s)$, which induces such a controller.  This model can be trained to imitate existing teacher algorithms via supervised learning, or by RL to optimize efficiency metrics to improve further.

Our first insight in this study is that setting up a neural network's input and output interfaces is critical for both generalization and efficiency. Just like in computer architecture design, the instruction set and how we access the data has great implications on the performance of the system.
Our second insight is that, in the space of algorithmic programs, supervised learning can at best match a teacher, but RL can help the model surpass it, as it has been observed in other domains such as games \cite{silver2016mastering, vinyals2019grandmaster}.

We show that our approach provides a promising paradigm for discovering new solutions for algorithmic tasks with learning, and our learned models can outperform comparable hand-coded %
programs in terms of efficiency in a range of tasks, including sorting, searching in ordered lists, and a version of the NP-complete 0/1 knapsack problem, whilst strongly generalizing to any tested input complexity.
In particular, we present a progression of models and interface designs that allows our model to learn $O(n^3)$, $O(n^2)$ and $O(n\log n)$ sorting algorithms, and outperforming our hand-coded bubble sort, insertion sort and quick sort algorithms.

The main contributions of this paper include:
    (1) learning neural networks for algorithmic tasks that generalize to instances of arbitrary length we tested on;
    (2) demonstrating the implications of model interfaces on generalization and efficiency, taking as inspiration CPU instruction sets;
    (3) using imitation and reinforcement learning to optimize efficiency metrics and discover new algorithms which can outperform strong teachers.

\section{Related work}
Training neural networks to solve algorithmic tasks has been a challenging domain for deep learning.  Our approach fits in the neural controller - interface framework \cite{graves2014neural,reed2015neural,zaremba2016learning} where a neural model interacts with external interfaces, like output or memory, through a range of instructions, that for example moves a write head, or changes a symbol in the output.  
The neural controller executes a sequence of such instructions until a task is solved or some termination condition is met.

\textbf{Optimizing for efficiency. }
We care not only about learning a ``correct'' algorithm, but also an ``efficient'' one.  As far as we are aware of, this is the first work in neural program induction that demonstrates improvement in efficiency over hand-coded algorithms, and we achieve this through reinforcement learning.
Program optimization has been studied in the area of superoptimization \cite{massalin1987superoptimizer}, which has seen success optimizing for example loop-free assembly code \cite{schkufza2013stochastic,bunel2016learning}.  However, this line of research so far still cannot handle more complex programs with control flow.

\textbf{Training on I/O examples vs on traces. }  %
Many prior works in neural program induction
learn only from input / output examples, and predicts the outputs directly without emitting a sequence of instructions \cite{kaiser2015neural,devlin2017neural}.
On the other side, when execution traces of intermediate steps are available, the learning task can be a lot simpler
and the learned models
more interpretable.
The Neural Programmer Interpreter (NPI) \cite{reed2015neural} model is a notable example of this.  NPI also proposed a way to learn and use functions, making the model more modular.
Follow up work like \cite{li2016neural,pierrot2019learning} explored ways to reduce the amount of supervision needed for training, and \cite{cai2017making} proposed to allow recursive function calls which can even lead to provably correct imitation of known algorithms.

\textbf{Differentiable vs non-differentiable instructions. }
Most existing work on neural program induction makes the instructions differentiable, such that the whole system is trainable end-to-end with gradient descent \cite{graves2014neural,joulin2015inferring,andrychowicz2016learning,neelakantan2015neural,graves2016hybrid}.  However, differentiable instructions are (1) typically more costly to compute,
and (2) tend to not generalize as well as discrete instructions \cite{kurach2015neural,joulin2015inferring,andrychowicz2016learning}.
In \cite{zaremba2015reinforcement,zaremba2016learning} the authors tried to learn with discrete instructions using RL, but reported difficulty in training.

\textbf{Recurrent vs non-recurrent memory. }
Most previous works also recommend using a form of recurrent neural network
as the controller, which has its own internal differentiable memory, while we propose to remove all the differentiable memory inside the model and if needed have memory in the environment instead, modified through discrete instructions.  These changes remove accumulating errors, which plays a significant role in attaining strong generalization when long action sequences (millions of steps in our experiments) are needed.  This observation is also shared by prior work \cite{giles1992learning}.  

\textbf{Comparison with NPI. } Our work is closely related to NPI \cite{reed2015neural,cai2017making},
with a few important differences.  Most notably: (1) our input representations are more carefully designed for the tasks; (2) we do not have recurrent memory in the model; (3) we have more programmatic semantics for function calls; and (4) we can use RL to learn models without supervised teacher traces, or RL + imitation to surpass the teacher.  These differences will be discussed in more detail later. 

\section{Method}
\vspace{-0.4em}
\subsection{Inputs, instructions and models}
\vspace{-0.4em}

We employ the neural program induction paradigm to learn neural networks to solve algorithmic tasks.  These networks are controllers that interact with external interfaces using a range of instructions.  The overall architecture of this neural controller is illustrated in \figref{fig:controller-interface}.

\begin{figure}[t]
\centering
\begin{minipage}[b]{.48\textwidth}
  \centering
  \includegraphics[width=0.8\textwidth]{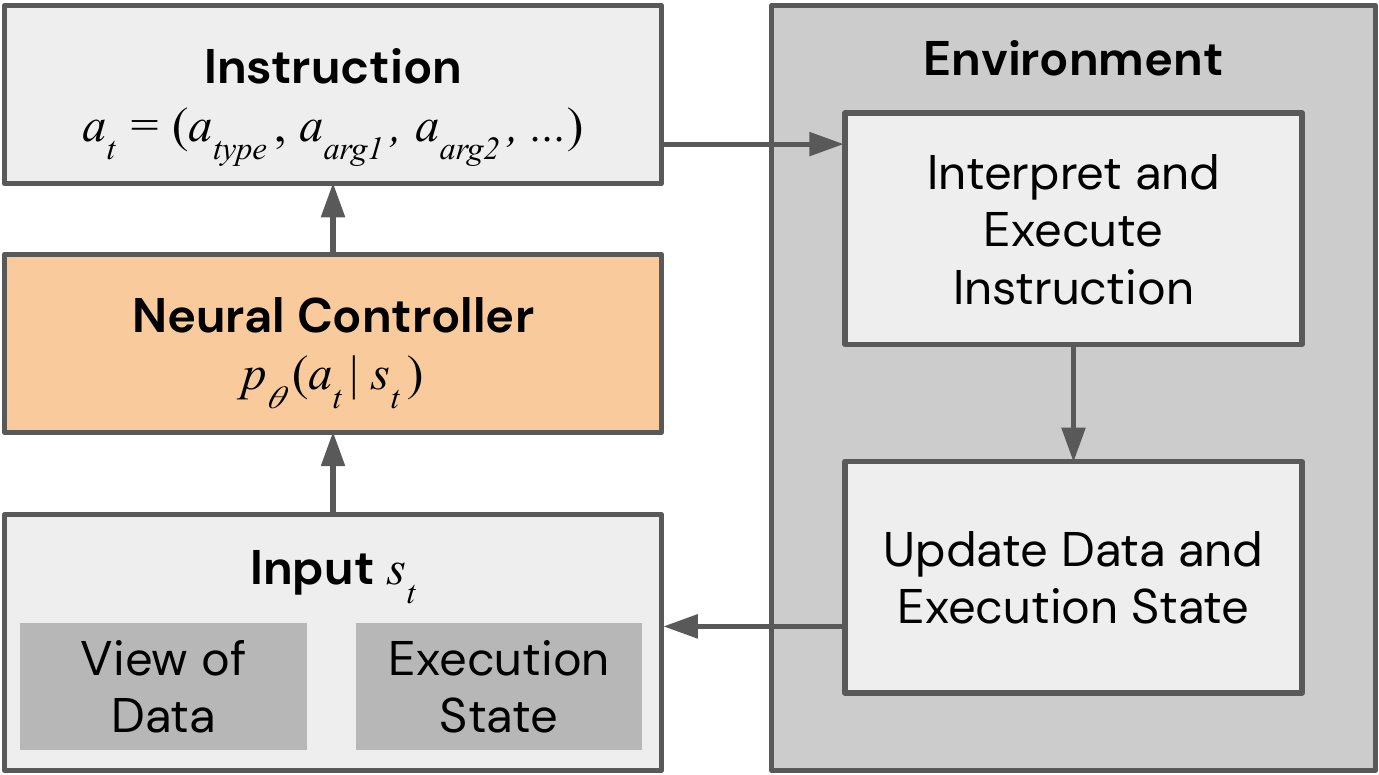}
  \vspace{-.5em}
  \caption{The overall architecture.}
  \label{fig:controller-interface}
\end{minipage}%
\hspace{1em}
\begin{minipage}[b]{.48\textwidth}
  \centering
  \includegraphics[width=0.95\textwidth]{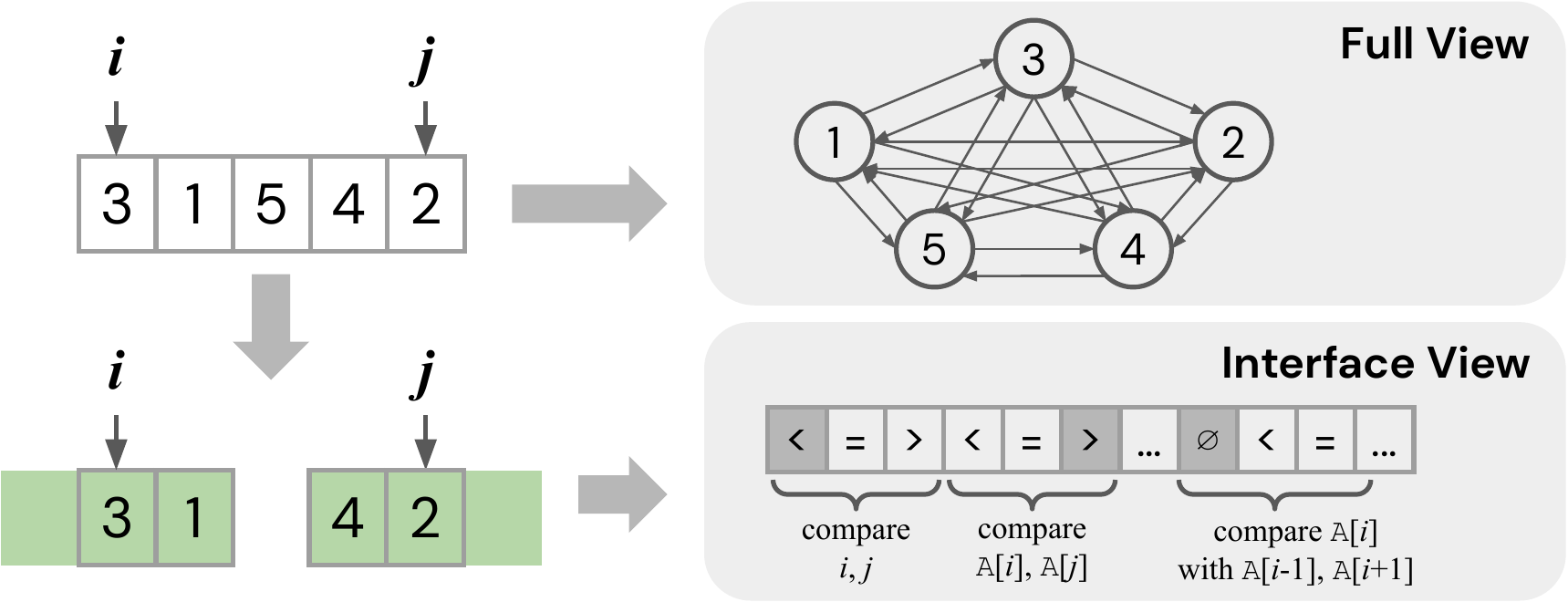}
  \vspace{-.5em}
  \caption{Full view of data vs. interface view.}
  \label{fig:data-views}
\end{minipage}
\vspace{-0.3em}
\end{figure}

\textbf{Input states. } Our neural controller takes a view of the data and execution state information as input $s_t$.  The exact form and content of the input is problem-dependent, conforming to the spec of the algorithm that we seek to learn and improve upon.
In the simplest case, this is a vector concatenation of all the available information in the problem setup.

What a neural controller can ``see'' as input plays a significant role in its generalization and efficiency.  In particular, if $s_t$ includes the full data instance at each step, processing $s_t$ would take at least $O(n)$ time, where $n$ is the instance size.  This might limit the learning of efficient algorithms as the cost of each step grows at least linearly with $n$.  On the other hand, models learned on this type of inputs would be sensitive to the instance length, and typically do not generalize well to instances of sizes not seen during training, a known shortcoming of statistical models such as neural networks \cite{vapnik2013nature}.

We found using a constant size input $s_t$ instead to be helpful for both efficiency and generalization.  This requires a constant size partial view of the data, and constant amount of execution state information, analogous to having a constant number of registers in a computer processing unit.

\textbf{Instruction set. } The instruction set defines what actions the neural controller can do.
Just like in computer architecture design, the instruction set plays a critical role for the system's success.
Notably, prior works on neural program induction
have not explored the implications of different choices of instruction sets.  For example, many works \cite{graves2014neural,graves2016hybrid} are based on a Turing machine model with read / write heads on memory and input / output tapes. The heads can only move one slot at a time, making it a linear time operation to move a head to a specified location.  Even though the Turing machine model is powerful, it is not efficient for most tasks.

We aim to design an instruction set that is rich enough to represent a variety of algorithms such that our neural programs live in an interesting space for learning, but also simple enough to have the potential to generalize across multiple tasks.  A typical instruction set contains instructions that change the data and may be task related, \eg~swap in a sorting task, but also instructions that manipulate execution state and control flow, that are task agnostic, \eg~moving a variable, or calling a function.

In our implementation, each instruction has a type $a_\text{type}$, as well as a list of arguments $a_\text{arg1}, a_\text{arg2}, ...$.
The arguments specify how and where the instruction should be executed.  This way of structuring instructions defines a structured action space.
Note that different instruction types may have a
different number of arguments,
and the arguments may have different types, e.g. binary or categorical.

The instruction set together with the input representations jointly determine the class of algorithms that are learnable by the neural controller, and we see this as a fruitful avenue for future research akin to how current instruction sets shaped microprocessors.

\textbf{The model. }
Our neural controller defines the mapping $a_t=f(s_t)$ through a policy distribution $p_\theta(a_t|s_t)$, with parameters $\theta$.
Based on the structure of $a_t$, $p_\theta(a_t|s_t)$ factorizes as the following:
\begin{equation}
    p_\theta(a_t|s_t) = p_\theta(a_\text{type}|s_t) p_\theta(a_\text{args}|a_\text{type}, s_t).
\end{equation}
The first part $p_\theta(a_\text{type}|s_t)$ is a categorical distribution over instruction types.  The second part 
is a type specific model of the argument list, 
and we use an autoregressive model for this, as
\begin{equation}
    p_\theta(a_\text{args}|a_\text{type}, s_t) = p_\theta(a_\text{arg1}|a_\text{type}, s_t) p_\theta(a_\text{arg2}|a_\text{arg1}, a_\text{type}, s_t) ...
\end{equation}
This decomposition of policy distribution for structured actions can also be found in \cite{vinyals2019grandmaster,openai2019dota}.
We consider 3 argument types: binary, integer and pointer (variable sized), %
modelled with Bernoulli, multinomial distributions and pointer mechanisms \cite{vinyals2015pointer}.
Depending on whether $s_t$ is fixed sized or variable sized, we use either MLPs or GNNs to process it. 
Each factor of $p_\theta(a_t|s_t)$, e.g. $p_\theta(a_\text{type}|s_t)$ or $p_\theta(a_\text{arg1}|a_\text{type}, s_t)$, is modelled using one of these neural networks, with proper output distributions.

\subsection{Supervised learning}

We first tried supervised learning on step-by-step traces from an existing, maybe hand-coded,
algorithm, \ie~the teacher, as in \cite{reed2015neural,li2016neural,cai2017making}. This paradigm is also known as behavior cloning (BC) in imitation learning. The teacher generates traces $(s_0, a_0, s_1, a_1, ...)$
following its %
policy $a_t = T(s_t)$, and the model is trained to minimize the standard negative log-likelihood objective $-\sum_t \log p_\theta(a_t|s_t)$.
Across all the experiments, our model can usually perfectly imitate the teacher within very few iterations.  However, models learned in such a way can at best match the teacher.

\subsection{Surpassing the teacher with RL}
\label{sec:model-rl}

We subsequently explored RL as a natural next step for optimizing the model to improve a certain efficiency metric, aiming to surpass the teacher, similar in spirit to \cite{silver2016mastering,vinyals2019grandmaster}.
The reward in RL is defined to encourage improving this metric. %
We target minimizing running time, by minimizing the number of steps required to solve an instance.
To do so, %
we define per-step reward $r_t=-c$ to be a constant negative penalty.
One episode terminates when a certain termination criterion is met, e.g. fully solving the task, or after taking a maximum allowed number of steps.
Therefore the less steps an agent takes in an episode, the higher the episode reward would be.  Additionally, we may design incremental rewards \cite{ng1999policy} to encourage the agent to learn the correct behavior faster.

We use $n$-step policy gradient
for learning, see e.g. \cite{sutton2018reinforcement}, and update the parameters $\theta$ and $\phi$ in the direction of
\begin{equation}
    \label{eqn:rl-objective}
    [G_t - V_\phi(s_t)]\nabla_\theta \log p_\theta(a_t|s_t) + \mu\nabla_\phi[G_t - V_\phi(s_t)]^2,
\end{equation}
where $G_t = r_t + \gamma r_{t+1} + \cdots + \gamma^{n-1}r_{t+n-1} + \gamma^{n}V_\phi(s_{t+n})$ is the bootstrapped $n$-step return, $\mu$ is a scalar weight, and $V_\phi(s)$ is a value estimate for state $s$, which could be implemented using a similar architecture as the policy model, or using simply a data-independent learned scalar parameter.  We tried both and found that using a single scalar provided more stability in learning in our experiments.

Learning with pure RL from scratch is difficult due to our large action space, long episodes, and sparse reward.
We therefore consider adding an auxiliary imitation loss to the policy gradient objective.
We collect traces $(s_t, a_t, r_t, s_{t+1}, a_{t+1}, r_{t+1}, ...)$ from the policy $p_\theta$ being trained, and then for each step $t$, query the teacher policy to get $a_t'=T(s_t)$.  The extra imitation loss is simply the negative log-likelihood of the teacher action under the current policy $-\lambda\log p_\theta(a_t'|s_t)$, similar to \cite{ross2010efficient,vinyals2019grandmaster}.  The scalar weight $\lambda$ is 
set to $\lambda=0.001$ in our experiments, which we found traded off rapid learning from the teacher, whilst retaining enough flexibility to improve upon it. %

\section{Problems, Interfaces and Experiments}
\label{sec:exp}
In this section, we explore the design choices of the input / output interfaces for the neural controller, and different learning paradigms on a range of tasks, building a powerful and generic instruction set which unlocks generalization and efficiency.  We use the sorting task as a driving case study and present a progression of designs that enable learning of strongly generalizable models, and improve time complexity from $O(n^3)$ to $O(n^2)$ and eventually $O(n \log n)$.  We then show the generality of this framework on two extra tasks, search in ordered lists and the 0/1 knapsack problem, where our learned models
outperform comparable hand-written programs.

The sorting task has been extensively used as a standard problem to test %
neural program induction models \cite{graves2014neural,vinyals2015order,reed2015neural,cai2017making,pierrot2019learning}.
Formally, we aim to sort a range of an array $\A[\low], \A[\low+1], ..., \A[\high]$ in ascending order. We denote $n=\high - \low+1$ as the number of elements to be sorted.

\subsection{Generalization / Correctness}

The generalization or correctness we aim for is mostly about generalization to instances of arbitrary sizes.  As far as we are aware of, most prior works tried but failed to achieve this, with \cite{cai2017making} as an exception.
For the sorting task we also care about generalization across numeric ranges, which previous work did not attempt.
As part of any sorting API, the user is asked to provide a comparator, which abstracts away the actual numeric values in the input, and only focuses on how the elements compare.  We implement the same solution here and allow the input to only contain the comparison results, \ie~is $\A[i] > \A[j]$, or $\A[i] = \A[j]$, or $\A[i] < \A[j]$.

\subsubsection{A model that does not generalize}
\label{sec:model-full-view}

The first attempt at this sorting task uses only a single instruction type ``Swap($i, j$)'',
and the input $s_t$ contains information about the full list at every step, through  %
the comparison results for all the $O(n^2)$ pairs of elements in \A.  We call this the \emph{full view} of data, as illustrated in \figref{fig:data-views}.

Since the size of the input is variable, we use graph neural networks (GNNs) \cite{battaglia2018relational,gilmer2017neural,li2015gated} to predict the arguments $i$ and $j$.
The input is structured as a fully connected graph of $n$ nodes for the $n$ elements in \A, with each edge containing a feature vector indicating the relative position of the pair and the comparison results.  
Note that each instruction is computed in $O(n^2)$ time as all $O(n^2)$ edges need to be processed, and sorting a list may require $O(n)$ swaps, therefore this interface can give us $O(n^3)$ algorithms at best.  More input and model details can be found in \appendixref{appendix:input-details} and \ref{appendix:model-details}.

We use a simple teacher policy in the supervised learning setting, which takes the same inputs and instruction set.
The teacher first checks whether the smallest value is put in the lowest index position, and if not swaps the value into that position, otherwise checks the second smallest value and so on.

For training, we sample $n$ uniformly from the range $10\le n \le 20$, and generate problem instances by uniformly perturbing the list $[0, 1, ..., n-1]$, and always set $\low=0$ and $\high=n-1$.

\tabref{tab:compare-views-results} (first row) shows the results in this setting, where we evaluate the best model on 100 random instances at each size.  We report the percentage of instances correctly solved, within a maximum of $n^2$ steps for each instance size.
Even though the learned model is generalizing in-distribution due in part to the strong generalization capabilities of GNNs, it fails to generalize when out-of-distribution to instances smaller or larger than the training instances.

\begin{table}[t]
    \centering
    \caption{Percentage of instances solved. %
    The cells {\setlength{\fboxsep}{0pt}\colorbox{gray!50}{highlighted in gray}} indicate the training size range.
    }
    \vspace{-0.5em}
    {\renewcommand{\arraystretch}{0.9}
    {\small
    \begin{tabular}{rccccccc|c}
    \toprule
        Instance size & 5 & 10 & 20 & 30 & 40 & 50 & 1000 & Complexity \\
    \midrule
        Full view & 72 & \graycell \textbf{100} & \graycell \textbf{100} & 33 & 0 & 0 & 0 & $O(n^3)$ \\
        Interface view & \textbf{100} & \graycell \textbf{100} & \graycell \textbf{100} & \textbf{100} & \textbf{100} & \textbf{100} & \textbf{100} & $O(n^2)$ \\
        \bottomrule
    \end{tabular}
    \label{tab:compare-views-results}
    }}
    \vspace{-1em}
\end{table}

\subsubsection{Constant size input and better instruction set lead to strong generalization}
\label{sec:bubble-insertion}

It is clear from the previous results that generalization across instance sizes is challenging, as also reported in many previous works \cite{reed2015neural,vinyals2015order,pierrot2019learning}.  %

Notably, none of the popular sorting algorithms decide which elements to swap by looking at the whole input at each execution step.  On the contrary, the decisions are typically made based on local evidence only. For example, in bubble sort (see \algref{alg:bubblesort} in \appendixref{appendix:scripted-sorting-algorithms}), we scan the input using an index variable $i$ and only compare $\A[i]$ with $\A[i+1]$ at each step to see if they can be swapped, without accessing other elements in the input.  Such locality brings two benefits: (1) the cost for each step becomes constant, rather than scaling with $n$ and (2) the algorithm becomes agnostic to the input size.
We implement this principle for the sorting task by employing a small set of $k$ (independent of $n$) index variables, and include in the input the information about how $\A[i]$ compares with $\A[i+1]$ and $\A[i-1]$ for each index variable $i$.
Furthermore, we also let the model know for each pair of variables $i,j$ how $i$ compares to $j$ and $\A[i]$ compares to $\A[j]$. %
We call this view of the data the \emph{interface view}, and is described in more detail in \appendixref{appendix:input-details} and illustrated in \figref{fig:data-views}.
We take $k=4$, with 4 index variables $v_1,v_2,v_3,v_4$ initialized as $v_1=v_3=\low$ and $v_2=v_4=\high$.

We also change and expand the instruction set, to 3 instruction types: 
SwapWithNext($i$) which swaps $\A[v_i]$ and $\A[v_i+1]$; MoveVar($i$, +1/-1) which assigns $v_i\gets \min\{v_i+1,\high\}$ or $\max\{v_i-1, \low\}$; and AssignVar($i,j$) which assigns $v_i\gets v_j$.
The arguments $i,j\in\{1,...,k\}$.  The MoveVar and AssignVar instructions can be found in \eg~the x86 instruction set, as \texttt{INC}/\texttt{DEC} and \texttt{MOV} instructions.

\textbf{Remark. } Such an input / output interface is sufficient to implement bubble and insertion sort.  However, the task of finding good algorithms within this space specified by the interface is still challenging and interesting, as (1) the space of possible programs is still huge, estimated as $28^{3.48\times 10^{10}}$ (see \appendixref{appendix:search-space-size}), such that brute force enumeration or search is hopeless; and (2) as will be shown later, this space contains well known algorithms, like bubble sort and insertion sort, but also algorithms that are even better which can be discovered by our models through learning (\secref{sec:surpass-teacher}).

In this setup, each step takes constant time, as the input has a fixed size. Therefore the time complexity of algorithms in this family is proportional to the number of steps required to solve an instance. 
We use MLPs to model $p_\theta(a|s)$.
The output has an extra distribution over instruction types $p(a_\text{type}|s_t)$, and then separate models for each type, again auto-regressive if there are more than one arguments.

\textbf{Strong generalization. } 
The second row in \tabref{tab:compare-views-results} shows the supervised learning results in this setting.
It is notable that the neural controller learned in this setting generalizes far beyond the training size range,
even on instances 100x larger than the training instances, requiring orders of magnitude more steps to solve, without a single failure case, indicating strong generalization. %
This difference in generalization performance from the full view results as shown in \tabref{tab:compare-views-results} signifies the impact of the changes in the input representation and the instruction set.

\subsection{Improving Efficiency}

\subsubsection{Functions enable divide-and-conquer}
\label{sec:functions}

The previous section shows that carefully designed input / output interfaces enable strong generalization.
However, from an efficiency perspective the algorithms that can be learned with such interfaces, at least for sorting, are still relatively slow $O(n^2)$ algorithms at best (see \appendixref{appendix:limit-of-bubble-insertion-interface} for a proof).

In this section, we further expand our instruction set by adding functions, a key ingredient for more efficient algorithms.
In computer architectures, a typical instruction set contains not only data movement / arithmetic instructions, but also control logic instructions, \eg~jump into a subroutine and return.  In programming languages, functions encourage reusability and modularity, and
enable algorithm design paradigms like divide-and-conquer,
the strategy behind many efficient algorithms, including notably the $O(n\log n)$ sorting algorithms like quick sort and merge sort.

We introduce two extra types of instructions, ``FunctionCall($id, l_1, ..., l_p, o_1, ..., o_p, r_1, ..., r_q$)'' and ``Return($l_1', ..., l_q'$)''.
These are generic instructions not tied
to a particular task.
The arguments
\begingroup %
\setlength{\columnsep}{8pt}
\setlength{\intextsep}{0pt}
\begin{wrapfigure}{r}{0.22\linewidth}
\lstdefinestyle{mystyle}{
    backgroundcolor=\color{gray!20},   
    keywordstyle=\color{blue},
    stringstyle=\color{codepurple},
    basicstyle=\ttfamily\footnotesize,
    keepspaces=true,                 
    showspaces=false,                
    showstringspaces=false,
    showtabs=false,                  
    tabsize=2
}
\lstset{style=mystyle}
\vspace{-6pt}
\begin{lstlisting}[language=Python]
def func(a, b):
  c = ...
  return c
x = func(y, z)
\end{lstlisting}
\vspace{-5pt}
\end{wrapfigure}
for these instructions specify which function to call and which values to be passed or returned, see \appendixref{appendix:func-call} for more details.
The example function call \texttt{x = func(y, z)} shown on the right can be implemented with FunctionCall(\texttt{func}, \texttt{a}, \texttt{b}, \texttt{y}, \texttt{z}, \texttt{x}), and \texttt{return c} with Return(\texttt{c}). %

\endgroup

The external (w.r.t. the neural controller) environment handles the semantics of function calls, %
by keeping track of a call stack.  The input state $s_t$ is also augmented, with the current function ID so the agent knows which function it is in.
Note that 
the controller still needs to figure out how to condition the policy based on the current function ID, or equivalently what to do inside each function. %

For the sorting task, we use 2 functions and %
also add an extra Swap($i,j$) instruction to allow swapping $\A[v_i]$ with $\A[v_j]$.  The input state $s_t$ contains all the comparison information described in \secref{sec:bubble-insertion}, the current function ID, and additionally the encoding of the previous action.
We can implement quick sort in this interface (see \appendixref{appendix:scripted-sorting-algorithms}).

In supervised learning, our neural controller can imitate quick sort perfectly, achieving the same performance as the teacher across all instance sizes, therefore capable of learning $O(n\log n)$ algorithms.

\subsubsection{Surpassing the teacher with RL}
\label{sec:surpass-teacher}

\begin{table}[t]
    \centering
    \caption{Average episode length for models learned using the bubble / insertion sort interface (top half), and using the extended quick sort interface (bottom half).  Note that \textbf{all models compared here achieve 
    100\% solve rate across all instance sizes}.
    *Evaluated with $10n^2$ steps max, instead of $n^2$, as not all episodes finish in $n^2$ steps.  Cells labeled as ``-'' didn't finish evaluation within 24 hours.
    }
    \vspace{-0.5em}
    \begingroup
    \setlength{\tabcolsep}{3pt}
    {\renewcommand{\arraystretch}{0.9}
    {\small
    \begin{tabular}{rcccccccccc}
    \toprule
    Instance size & 5 & 10 & 20 & 30 & 50 & 100 & 200 & 500 & 1000 & 10000 \\
    \midrule
    Bubble sort & 13.5 & \graycell 68.0 & \graycell 293.6 & 677.8 & 1,874.7 & 7,527.0 & 30,046.8 & 187,506.2 & 750,519.3 & - \\
    Insertion sort & 13.7 & \graycell 53.4 & \graycell 208.6 & 475.3 & 1,275.7 & 5,077.5 & 20,134.3 & 125,136.8 & 501,051.7 & - \\
    RL from scratch & 8.7 & \graycell 49.7 & \graycell 252.0 & 617.5 & 1,763.7 & 7,320.5 & 29,623.0 & 186,305.1 & 748,784.5 & - \\
    RL + imitation & \textbf{8.2} & \graycell\textbf{44.3} & \graycell\textbf{190.7} & \textbf{446.6} & \textbf{1,228.5} & \textbf{4,981.2} & \textbf{19,939.2} & \textbf{124,643.0} & \textbf{500,057.0} & - \\
    \midrule
    Quick sort & 
    27.4* & 87.5* & 241.8 & \graycell 399.9 & \graycell 788.8 & 1,840.2 & 4,217.1 & 12,519.4 & 27,633.5 & 368,338.6 \\
    RL from scratch & \bf 13.5* & \bf 56.0* & \bf 220.6 & \graycell 495.5 & \graycell 1,315.6 & 5,166.5 & 20,321.4 & 125,622.0 & 502,035.5 & - \\
    RL + imitation & 24.8* & 79.9* & 223.7 & \graycell \bf 377.2 & \graycell \bf 728.3 & \bf 1,717.4 & \bf 3,914.4 & \bf 11,578.2 & \bf 25,635.0 & \bf 338,947.9 \\
    \bottomrule
    \end{tabular}
    }}
    \endgroup
    \label{tab:n2-and-nlogn-results}
    \vspace{-0em}
\end{table}

With RL the neural controller has the potential to surpass the teacher by further optimizing a performance metric.
However, exploration is a hard challenge for RL, in particular with our large structured action space, long episodes and sparse reward.
We explored adding an imitation loss and a form of reward shaping \cite{ng1999policy} to alleviate this issue.  For the sorting task, the shaping reward we used rewards each swap that increases the ``orderdness'' of the array (see \appendixref{appendix:shaping-reward}).

Our best models are trained without reward shaping, through RL + imitation alone (\secref{sec:model-rl}). We do not initialize the model with the one from supervised learning.
However, reward shaping is needed when not using an imitation loss.  We observed that reward shaping helps learning the \emph{correct} behavior much faster than with the sparse reward $r_t=-c$.  But it is also easier for the model to get trapped in a local optimum in terms of \emph{efficiency}.

In \tabref{tab:n2-and-nlogn-results}, we present RL results with and without imitation under the two interfaces introduced in \secref{sec:bubble-insertion} and \secref{sec:functions}.
The results come from the best models as evaluated on the training range from a sweep over hyperparameters
(\appendixref{appendix:model-details}).
To learn quick sort, we have to train on larger instances, in our case $30\le n\le 50$, as simple $O(n^2)$ algorithms are more efficient on small instances. \appendixref{appendix:more-results} contains more experiment results as well as a few typical training curves.

\textbf{Surpassing the teacher. } %
We can see that RL + imitation can outperform the teacher policy (insertion sort for the first setting, and quick sort for the second setting) across all instance size ranges, while maintaining 100\% solve rate on all instances tested, indicating the discovery of new algorithms.  We include videos in the supplementary material showing the execution traces of the learned algorithms compared with the teachers, and observe qualitatively different behaviors.

\textbf{Learning from scratch is possible. } The RL from scratch without imitation results in \tabref{tab:n2-and-nlogn-results} shows that pure RL is also possible.  Notably the learned models also achieve 100\% solve rate, even though the efficiency results are worse than RL + imitation.
A nice property about this setup is that no trace-level supervision is needed, therefore it can be used in a domain where no prior solutions exist.

\subsection{Generality}

\begin{table}[t]
    \centering
    \caption{Results for searching in ordered lists, reporting average episode length, evaluated on
    100 instances for each instance size.  All approaches achieve 100\% solve rate.
    Both learned agents outperform hand-coded algorithms in the training range of instances ({\setlength{\fboxsep}{0pt}\colorbox{gray!50}{high-lighted in gray}}) on average.}
    \vspace{-0.5em}
    {\renewcommand{\arraystretch}{0.9}
    {\small
    \begin{tabular}{rccccccccc}
    \toprule
        Instance Size & 5 & 10 & 20 & 30 & 50 & 100 & 200 & 500 & 1000  \\
        \midrule
        Linear Search & 2.2 & 4.6 & 9.7 & \graycell 14.2 & \graycell 23.3 & 50.0 & 106.7 & 220.7 & 446.6 \\
        Binary Search & 2.6 & 3.9 & 6.3 & \graycell 7.5 & \graycell 9.2 & 11.7 & \bf 14.6 & \bf 19.1 & \bf 21.8 \\
        \midrule
        RL from scratch & \bf 1.3 & \bf 2.5 & 4.3 & \graycell 5.9 & \graycell 9.9 & 15.5 & 29.7 & 85.7 & 168.6 \\
        RL + imitation & 1.4 & 2.6 & \bf 4.1 & \graycell \bf 4.7 & \graycell \bf 7.0 & \bf 10.6 & 18.4 & 41.5 & 85.8 \\ 
        \bottomrule
    \end{tabular}
    }}
    \label{tab:listsearch}
\end{table}

\begin{table}[t]
    \centering
    \caption{Results for the 0/1 Knapsack problem, measured for different instance sizes and different step budgets.
    Reporting the average reward achieved over 100 instances, higher is better.
    Overall, our learned neural programs with 5x less budget perform as well as the DFS strategy.}
    \vspace{-0.5em}
    \begingroup
    \setlength{\tabcolsep}{4pt}
    {\renewcommand{\arraystretch}{0.9}
    {\small
    \begin{tabular}{r|r|ccccccccccc}
    \toprule
        Budget & Size & 2 & 4 & 6 & 8 & 10 & 20 & 40 & 80 & 160 & 320 & 640 \\
        \midrule
        \multirow{2}{*}{20x size}& DFS & 0.45 & \graycell 1.25 & \graycell 1.82 & \graycell 2.47 & 3.14 & 5.94 & 11.10 & 21.67 & 42.11 & 82.01 & 163.42 \\
        & RL & 0.45 & \graycell \bf 1.26 & \graycell \bf 2.03 & \graycell \bf 2.83 & \bf 3.51 & \bf 6.51 & \bf 11.78 & \bf 22.40 & \bf 43.27 & \bf 83.07 & \bf 164.56 \\
        \midrule
        \multirow{2}{*}{100x size}& DFS & 0.45 & \graycell 1.26 & \graycell \bf 2.04 & \graycell 2.81 & 3.43 & 6.43 & 11.75 & 22.39 & 42.95 & 82.82 & 164.22 \\
        & RL & 0.45 & \graycell 1.26 & \graycell 2.03 & \graycell \bf 2.86 & \bf 3.70 & \bf 6.98 & \bf 12.41 & \bf 23.15 & \bf 44.16 & \bf 84.05 & \bf 165.60 \\
        \bottomrule
    \end{tabular}
    }}
    \endgroup
    \label{tab:knapsack}
\end{table}

The principles introduced in the previous sections are not sorting specific, and can be applied to many more algorithmic tasks.  To demonstrate this, in this section we introduce two additional tasks, searching in ordered lists, and the NP-complete 0/1/ knapsack problem.

\textbf{Identical input interface and similar instruction set applied on a different task: searching in ordered lists. }  This task involves searching for a query element $q$, in a sorted list $\A[0]\le ... \le \A[n-1]$.  A correct algorithm would return an index $i$ such that $\A[i]=q$, or report element not found if none of the elements in the list equals $q$.  A linear scan from left to right can solve this task in $O(n)$ time, but the well known binary search algorithm can reduce the number of steps needed to $O(\log n)$.

For this task we reuse the comparison-based input interface with index variables used in sorting, and also reuse the instructions MoveVar and AssignVar.  Additionally, we add: (1) AssignMid($i, j, k$) instruction that assigns $v_i\gets \lfloor (v_j+v_k)/2\rfloor$, (2) Found($i$) instruction that reports element $q$ found at index $v_i$, and (3) NotFound() instruction that reports $q$ not found in \A.  The Found and NotFound instructions terminate the episode.  This setup specifies a family of controllers that includes both the $O(n)$ linear scan algorithm as well as the $O(\log n)$ binary search algorithm.

Since imitation learning trivially learns to imitate the teachers, in \tabref{tab:listsearch} we compare the models learned through RL against the hand-coded teachers.  In the RL + imitation setting we used binary search as the teacher.  The reward for this task is $r_t=-c$, and a large negative reward is applied if an episode terminates with a wrong result, \ie~reporting Found or NotFound incorrectly.  We can again see that (1) the learned models achieve perfect generalization,
and (2) RL can help learning models that surpass the teacher,
and RL + imitation performs even better.  Notably, near the training range, our learned models exhibit a time complexity closer to $O(\log n)$ which becomes linear when far out of training range, indicating some level of specialization.

\textbf{Different inputs and different instruction set: 0/1 knapsack problem. }  This task is a classic NP-complete problem.  In each problem instance, there are $n$ items each with weight $w_i$ and value $v_i$, and we have to pick items to put into a knapsack with maximum weight capacity of $W$ such that the items in the knapsack has maximum value.  We consider real valued $w_i$ and $v_i$, both sampled from uniform distribution $U[0,1]$, and set $W=\frac{1}{2}\sum_i w_i$.  Since solving each instance exactly requires exponential time, we consider instead maximizing value within a fixed step budget.

For this task we use the simple depth-first search (DFS) as a teacher policy, which exhaustively enumerates all the $2^n$ possibilities.  We design an environment that keeps track of the current item index $i$, the current list of items in the knapsack, their total weight and value, and the best value achieved so far.  We use a set of instructions including: Put() which puts the current item $i$ into the knapsack, Pop() which pops the item $i$ from the knapsack, MoveVar(+1/-1) that assigns $i\gets i+1$ or $i-1$, Knapsack() which is a function call that calls the knapsack function recursively, and Return() that returns from the current level of Knapsack call.  Such an interface defines a family of controllers that includes the simple DFS algorithm as a special case.

\tabref{tab:knapsack} compares the model learned with RL against DFS.  We use $r_t=($best value after $a_t) -($best value before $a_t)$,
such that the total reward of an episode is the value of the best solution found.  The models are trained on instances of size 4-8, with a maximum step budget of 200.  In this case, RL from scratch (reported in the table) finds a better solution than RL + imitation, and the learned model can find solutions as good as DFS but with 5x less step budget, even well beyond the training range.

\section{Limitations and future work}
In this paper we studied the design choices that enable learning strongly generalizable and efficient neural program induction models that follows the $a=f(s)$ framework.  These proposals, albeit successful, still have limitations that provide valuable directions for future work.

Notably, computing $a=f(s)$ can be more expensive on current CPUs than executing typical computation employed in the algorithms studied here. We thus hope that this research will motivate future CPUs to have ``Neural Logic Units'' to implement such functions $f$ fast and efficiently, effectively extending their instruction set, and making such approaches feasible.

Also, in our framework, a user still needs to design an input / output interface and ideally provide a solution to employ the RL + imitation setting so that our approach or future variants can improve upon the teacher. Possible solutions include designing a ``universal'' instruction set that covers a wide range of tasks, similar to how instruction sets have been designed for CPUs. Further improvements in RL could also obviate the need for providing an initial solution to bootstrap learning %
and optimization. %

Another area that our approach can potentially shine is to automatically adapt an algorithm or neural program to a new data distribution through learning without manual tuning, which presents new challenges but also exciting opportunities for the research community.

\section*{Broader Impact}
The idea and approaches studied in this paper could have broad impact over how we optimize algorithms and programs.  Developed further, this research could help us find more efficient algorithms for solving a variety of challenges that have practical value.  For example, finding a new sorting algorithm that is more efficient than popular standard library tools can create great value (even a small percentage improvement can have great impact) as the new algorithm could be deployed everywhere easily. Another example is finding a new and more efficient algorithm for solving NP-hard problems, for example traveling salesperson problem or mixed integer programming, which can directly benefit transport or logistics industry.  On the other side, there is potential for technologies like this to be misused, for example optimizing an algorithm for breaking a security system.  We believe there is still a long way to go before the technology is mature enough to be widely used, but understanding and evaluating such risks is paramount to our research.

\section*{Acknowledgements}
The authors would like to thank Nando de Freitas, Xujie Si, Alex Gaunt, Scott Reed and Feryal Behbahani, Yuhuai Wu and many others at DeepMind for their helpful discussions and comments.

\bibliography{refs}

\begin{thebibliography}{10}

\bibitem{andrychowicz2016learning}
Marcin Andrychowicz and Karol Kurach.
\newblock Learning efficient algorithms with hierarchical attentive memory.
\newblock {\em arXiv preprint arXiv:1602.03218}, 2016.

\bibitem{battaglia2018relational}
Peter~W Battaglia, Jessica~B Hamrick, Victor Bapst, Alvaro Sanchez-Gonzalez,
  Vinicius Zambaldi, Mateusz Malinowski, Andrea Tacchetti, David Raposo, Adam
  Santoro, Ryan Faulkner, et~al.
\newblock Relational inductive biases, deep learning, and graph networks.
\newblock {\em arXiv preprint arXiv:1806.01261}, 2018.

\bibitem{bunel2016learning}
Rudy Bunel, Alban Desmaison, M~Pawan Kumar, Philip~HS Torr, and Pushmeet Kohli.
\newblock Learning to superoptimize programs.
\newblock {\em arXiv preprint arXiv:1611.01787}, 2016.

\bibitem{cai2017making}
Jonathon Cai, Richard Shin, and Dawn Song.
\newblock Making neural programming architectures generalize via recursion.
\newblock {\em arXiv preprint arXiv:1704.06611}, 2017.

\bibitem{devlin2017neural}
Jacob Devlin, Rudy~R Bunel, Rishabh Singh, Matthew Hausknecht, and Pushmeet
  Kohli.
\newblock Neural program meta-induction.
\newblock In {\em Advances in Neural Information Processing Systems}, pages
  2080--2088, 2017.

\bibitem{gaunt2016terpret}
Alexander~L Gaunt, Marc Brockschmidt, Rishabh Singh, Nate Kushman, Pushmeet
  Kohli, Jonathan Taylor, and Daniel Tarlow.
\newblock Terpret: A probabilistic programming language for program induction.
\newblock {\em arXiv preprint arXiv:1608.04428}, 2016.

\bibitem{giles1992learning}
C~Lee Giles, Clifford~B Miller, Dong Chen, Hsing-Hen Chen, Guo-Zheng Sun, and
  Yee-Chun Lee.
\newblock Learning and extracting finite state automata with second-order
  recurrent neural networks.
\newblock {\em Neural Computation}, 4(3):393--405, 1992.

\bibitem{gilmer2017neural}
Justin Gilmer, Samuel~S Schoenholz, Patrick~F Riley, Oriol Vinyals, and
  George~E Dahl.
\newblock Neural message passing for quantum chemistry.
\newblock In {\em Proceedings of the 34th International Conference on Machine
  Learning-Volume 70}, pages 1263--1272. JMLR. org, 2017.

\bibitem{graves2014neural}
Alex Graves, Greg Wayne, and Ivo Danihelka.
\newblock Neural turing machines.
\newblock {\em arXiv preprint arXiv:1410.5401}, 2014.

\bibitem{graves2016hybrid}
Alex Graves, Greg Wayne, Malcolm Reynolds, Tim Harley, Ivo Danihelka, Agnieszka
  Grabska-Barwi{\'n}ska, Sergio~G{\'o}mez Colmenarejo, Edward Grefenstette,
  Tiago Ramalho, John Agapiou, et~al.
\newblock Hybrid computing using a neural network with dynamic external memory.
\newblock {\em Nature}, 538(7626):471, 2016.

\bibitem{joulin2015inferring}
Armand Joulin and Tomas Mikolov.
\newblock Inferring algorithmic patterns with stack-augmented recurrent nets.
\newblock In {\em Advances in neural information processing systems}, pages
  190--198, 2015.

\bibitem{kaiser2015neural}
{\L}ukasz Kaiser and Ilya Sutskever.
\newblock Neural gpus learn algorithms.
\newblock {\em arXiv preprint arXiv:1511.08228}, 2015.

\bibitem{kurach2015neural}
Karol Kurach, Marcin Andrychowicz, and Ilya Sutskever.
\newblock Neural random-access machines.
\newblock {\em arXiv preprint arXiv:1511.06392}, 2015.

\bibitem{li2016neural}
Chengtao Li, Daniel Tarlow, Alexander~L Gaunt, Marc Brockschmidt, and Nate
  Kushman.
\newblock Neural program lattices.
\newblock 2016.

\bibitem{li2015gated}
Yujia Li, Daniel Tarlow, Marc Brockschmidt, and Richard Zemel.
\newblock Gated graph sequence neural networks.
\newblock {\em arXiv preprint arXiv:1511.05493}, 2015.

\bibitem{massalin1987superoptimizer}
Henry Massalin.
\newblock Superoptimizer: a look at the smallest program.
\newblock In {\em ACM SIGARCH Computer Architecture News}, volume~15, pages
  122--126. IEEE Computer Society Press, 1987.

\bibitem{neelakantan2015neural}
Arvind Neelakantan, Quoc~V Le, and Ilya Sutskever.
\newblock Neural programmer: Inducing latent programs with gradient descent.
\newblock {\em arXiv preprint arXiv:1511.04834}, 2015.

\bibitem{ng1999policy}
Andrew~Y Ng, Daishi Harada, and Stuart Russell.
\newblock Policy invariance under reward transformations: Theory and
  application to reward shaping.
\newblock In {\em ICML}, volume~99, pages 278--287, 1999.

\bibitem{openai2019dota}
OpenAI, Christopher Berner, Greg Brockman, Brooke Chan, Vicki Cheung,
  Przemys{\l}aw Debiak, Christy Dennison, David Farhi, Quirin Fischer, Shariq
  Hashme, Chris Hesse, et~al.
\newblock Dota 2 with large scale deep reinforcement learning.
\newblock {\em arXiv preprint arXiv:1912.06680}, 2019.

\bibitem{pierrot2019learning}
Thomas Pierrot, Guillaume Ligner, Scott Reed, Olivier Sigaud, Nicolas Perrin,
  Alexandre Laterre, David Kas, Karim Beguir, and Nando de~Freitas.
\newblock Learning compositional neural programs with recursive tree search and
  planning.
\newblock {\em arXiv preprint arXiv:1905.12941}, 2019.

\bibitem{reed2015neural}
Scott Reed and Nando De~Freitas.
\newblock Neural programmer-interpreters.
\newblock {\em arXiv preprint arXiv:1511.06279}, 2015.

\bibitem{ross2010efficient}
St{\'e}phane Ross and Drew Bagnell.
\newblock Efficient reductions for imitation learning.
\newblock In {\em Proceedings of the thirteenth international conference on
  artificial intelligence and statistics}, pages 661--668, 2010.

\bibitem{schkufza2013stochastic}
Eric Schkufza, Rahul Sharma, and Alex Aiken.
\newblock Stochastic superoptimization.
\newblock In {\em ACM SIGPLAN Notices}, volume~48, pages 305--316. ACM, 2013.

\bibitem{silver2016mastering}
David Silver, Aja Huang, Chris~J Maddison, Arthur Guez, Laurent Sifre, George
  Van Den~Driessche, Julian Schrittwieser, Ioannis Antonoglou, Veda
  Panneershelvam, Marc Lanctot, et~al.
\newblock Mastering the game of go with deep neural networks and tree search.
\newblock {\em nature}, 529(7587):484, 2016.

\bibitem{sutton2018reinforcement}
Richard~S Sutton and Andrew~G Barto.
\newblock {\em Reinforcement learning: An introduction}.
\newblock 2018.

\bibitem{vapnik2013nature}
Vladimir Vapnik.
\newblock {\em The nature of statistical learning theory}.
\newblock Springer science \& business media, 2013.

\bibitem{vinyals2019grandmaster}
Oriol Vinyals, Igor Babuschkin, Wojciech~M Czarnecki, Micha{\"e}l Mathieu,
  Andrew Dudzik, Junyoung Chung, David~H Choi, Richard Powell, Timo Ewalds,
  Petko Georgiev, et~al.
\newblock Grandmaster level in starcraft ii using multi-agent reinforcement
  learning.
\newblock {\em Nature}, 575(7782):350--354, 2019.

\bibitem{vinyals2015order}
Oriol Vinyals, Samy Bengio, and Manjunath Kudlur.
\newblock Order matters: Sequence to sequence for sets.
\newblock {\em arXiv preprint arXiv:1511.06391}, 2015.

\bibitem{vinyals2015pointer}
Oriol Vinyals, Meire Fortunato, and Navdeep Jaitly.
\newblock Pointer networks.
\newblock In {\em Advances in Neural Information Processing Systems}, pages
  2692--2700, 2015.

\bibitem{zaremba2016learning}
Wojciech Zaremba, Tomas Mikolov, Armand Joulin, and Rob Fergus.
\newblock Learning simple algorithms from examples.
\newblock In {\em International Conference on Machine Learning}, pages
  421--429, 2016.

\bibitem{zaremba2015reinforcement}
Wojciech Zaremba and Ilya Sutskever.
\newblock Reinforcement learning neural turing machines-revised.
\newblock {\em arXiv preprint arXiv:1505.00521}, 2015.

\end{thebibliography}
\bibliographystyle{plain}

\clearpage

\appendix

\part{Appendix}
\parttoc

\section{Detailed description of input feature representations}
\label{appendix:input-details}

In this section we describe the detailed input representations for each of the tasks and interfaces we consider.  As discussed in the main paper, what to put into the input state $s_t$ makes a big difference in a model's performance, affecting both correctness and efficiency.

\subsection{Sorting}

\subsubsection{Full view of data (\secref{sec:model-full-view})}

In this setting, input $s_t$ is a graph $G=(V, E)$, where we have one node for each element $\A[i]$ in the range to be sorted, where $\low \le i \le \high$, and a directed edge for each pair of indices $(i, j), \forall \low \le i, j\le \high, i\ne j$.  The nodes and edges are both attributed with a feature vector.

For each edge we have a 2-dimensional feature vector
\begin{equation}
[\text{sign}(i - j), \text{sign}(\A[i] - \A[j])],
\end{equation}
where
\begin{equation}
    \text{sign}(x) = \left\{\begin{array}{rl}
    1, & \text{if } x > 0 \\
    0, & \text{if } x = 0 \\
    -1, & \text{if } x < 0
    \end{array}\right..
\end{equation}
These edge features contain all the information about the data.  The node features are chosen to be non-informative, and we used a single constant value of 1 (1-dimensional feature) as node features.  The node features have a shape of $|V|\times 1$, and the edge features have a shape of $|E|\times 2$.

\subsubsection{Bubble / insertion sort interface (\secref{sec:bubble-insertion})}
\label{appendix:bubble-insertion-interface-details}

In this setting, the input $s_t$ only contains a partial view of the data, \ie~information about $\A$, as well as some execution states, \eg~information about where the pointers are.

As described in the main paper, in this setting the environment keeps track of $k$ index variables $v_1, ..., v_k$, and the state $s_t$ is a vector concatenation of all the following:
\begin{itemize}
    \item For each $1\le i,j\le k$ and $i < j$, information about how $v_i$ compares with $v_j$ and how $\A[v_i]$ compares with $\A[v_j]$ (6-dimensional vector):
    \begin{equation*}
        [\I[v_i < v_j], \I[v_i = v_j], \I[v_i > v_j], \I[\A[v_i] < \A[v_j]], \I[\A[v_i] = \A[v_j]], \I[\A[v_i] > \A[v_j]]],
    \end{equation*}
    where $\I[.]$ is the indicator function.
    \item For each $1\le i\le k$, information about how $\A[v_i]$ compares with $\A[v_i-1]$, represented as a 4-dimensional 1-hot vector. The table below lists each of the 4 cases (corresponding to 4 dimensions of this vector):
    
    \begin{tabular}{c|c}
        if $v_i - 1 < \low$ & else  \\
        \midrule
        1 & 0 \\
        0 & $\I[\A[v_i] > \A[v_i - 1]]$ \\
        0 & $\I[\A[v_i] = \A[v_i - 1]]$ \\
        0 & $\I[\A[v_i] < \A[v_i - 1]]$
    \end{tabular}
    
    and how $\A[v_i]$ compares with $\A[v_i+1]$ (the table lists each of the 4 dimensions of this vector):
    
    \begin{tabular}{c|c}
        if $v_i + 1 > \high$ & else  \\
        \midrule
        0 & $\I[\A[v_i] > \A[v_i + 1]]$ \\
        0 & $\I[\A[v_i] = \A[v_i + 1]]$ \\
        0 & $\I[\A[v_i] < \A[v_i + 1]]$ \\
        1 & 0
    \end{tabular}
\end{itemize}
For our setting with $k=4$, the input $s_t$ is a vector of size $6k(k-1)/2 + 4k + 4k = 68$.

\subsubsection{Quick sort interface (\secref{sec:functions})}

The input $s_t$ for this extended interface inherits all of the inputs from the bubble / insertion interface, with two additional parts: (1) the encoding of the current function ID; and (2) the encoding of the previous action.

The encoding of the current function ID is critical for the model to condition its behavior on.  The model does different things conditioned on different function IDs, therefore implementing different functions.  The current function ID is encoded as a one-hot vector of size $F+1$ which is the number of allowed functions in the environment $F$, plus 1 for the out-most scope which is not in any function.

The encoding of the previous action is useful for disambiguating different states.  Note that since the input $s_t$ contains only a partial view of the data and execution state information, many different data instances and execution states would be mapped to the same input state $s_t$.  This is an important design decision that affects the generalization and efficiency performance of our model.  On the other hand, over-restricting the input $s_t$ may cause too many states to be mixed into one, which may restrict the class of algorithms that could be learned.  In this setting, we found that we could add the encoding of the previous action to disambiguate all the necessary situations to properly implement quick sort in our framework.

The vector encoding of an action has a few parts: (1) encoding of the action type, we use a one-hot encoding for this; and (2) encoding of the action arguments.  We have one set of action argument encodings for each possible action type, and fill in 0s when the action type is not the actual action type.  For each action argument:
\begin{itemize}
    \item If the type is boolean, we encode it as 1 if it is True and 0 otherwise.
    \item If the type is integer or pointer, we encode it as a one-hot vector with size the same as the number of possible values.
\end{itemize}

The encoding size for each action type in our sorting setting ($k=4, F=2$) is listed in \tabref{tab:encoding-size-quick-sort}.
\begin{table}[h]
\centering
\caption{Encoding size for the arguments in each action type in the quick sort interface.}
\label{tab:encoding-size-quick-sort}
\begin{tabular}{c|c}
    \toprule
    Action Type & Encoding Size \\
    \midrule
    SwapWithNext($i$) & $4$ \\
    MoveVar($i$, +1/-1) & $4 + 1$ \\
    AssignVar($i, j$) & $4 + 4$ \\
    FunctionCall($id, l_1, l_2, o_1, o_2, r_1$) & $2 + 4 \times 5$ \\
    Return($l_1'$) & $4$ \\
    Swap($i,j$) & $4 + 4$ \\
    \bottomrule
\end{tabular}
\end{table}

Summing up all the action argument embedding sizes for each action type, we get a total size of $51$.  The total size of the input $s_t$ is therefore $68$ (features from the bubble / insertion interface) $+ 3$ (current function ID one-hot) $+ 6$ (action type one-hot)$ + 51$ (action arguments encoding)$ + 1$ (whether previous action is None)$ = 129$.  Note that we added an extra bit to indicate if the action is None, which is necessary when, \eg~at the start of the episode, or when entering a function.

\subsection{Searching in ordered list}

For this task we use an interface very similar to sorting.  The environment again keeps track of $k$ index variables $v_1, ..., v_k$, and we use the same comparison features used in the bubble / insertion sort interface.  On top of this, we also add two extra parts to the input:
\begin{itemize}
    \item Comparison between the query $q$ and each of $\A[v_i]$, \ie~for each $1\le i\le k$, a vector of size 3:
    \begin{equation*}
        [\I[q < \A[v_i]], \I[q = \A[v_i]], \I[q > \A[v_i]]].
    \end{equation*}
    \item Encoding of the previous action, with the action types and encoding sizes for this environment listed in \tabref{tab:actions-and-encoding-size-list-search}.
\end{itemize}

\begin{table}[h]
    \centering
    \caption{Action types and action argument encoding sizes for the searching in ordered lists task.}
    \label{tab:actions-and-encoding-size-list-search}
    \begin{tabular}{c|c}
    \toprule
        Action Type & Encoding Size \\
        \midrule
        MoveVar($i$, +1/-1) & $4 + 1$ \\
        AssignVar($i, j$) & $4 + 4$ \\
        AssignMid($i, j, k$) & $4\times 3$ \\
        Found($i$) & $4$ \\
        NotFound() & 0 \\
    \bottomrule
    \end{tabular}
\end{table}

The total size of the input $s_t$ is therefore $68$ (bubble / insertion interface) $+3\times 4$ (comparison with query)$+5$ (action type encoding) $+(4+1)+(4+4)+4\times 3 + 4$ (action argument encoding)$+1$ (whether previous action is None) $=115$.

\subsection{0/1 knapsack problem}

For this task, the interface is designed so that the basic depth-first search policy can be implemented.  The environment keeps track of a single index variable $i$, and the input $s_t$ contains:
\begin{itemize}
    \item The following list of features summarizing the search progress:
    
    \begin{tabular}{c|p{0.43\columnwidth}}
    \toprule
    Feature & Note \\
    \midrule
    sign($i$) & If $i$ is beyond 0 \\
    sign($i-n$) & If $i$ is beyond $n$ (number of items) \\
    $\I[i\in sol]$ & If item $i$ is in the current solution \\
    $\I[w \le W]$ & If the current weight is smaller than capacity \\
    $\I[v + v_\text{rest} > v^*]$ & If the current solution has potential to do better than the best solution \\
    $\I[w + \sum w_\text{rest} \le W]$ & If current weight is light enough to fit all the rest of the items \\
    $\I[w + w_\text{min} \le W]$ & If at least one more item can fit into the current solution \\
    \bottomrule
    \end{tabular}
    
    where $w$ is the weight of the current knapsack, $v$ is the value of the current solution, \ie~sum of the values for all the items already in the knapsack, $v^*$ is the value of the best solution found so far in the episode, and we define
    \begin{align*}
        v_\text{rest} &= \left\{\begin{array}{ll}
            \sum_{j=i+1}^n v_j & \text{if $i\in sol$} \\
            \sum_{j=i}^n v_j & \text{if $i\notin sol$}
        \end{array}\right. \\
        w_\text{rest} &= \left\{\begin{array}{ll}
            \sum_{j=i+1}^n w_j & \text{if $i\in sol$} \\
            \sum_{j=i}^n w_j & \text{if $i\notin sol$}
        \end{array}\right. \\
        w_\text{min} &= \left\{\begin{array}{ll}
            \min\{w_{i+1}, ..., w_n\} & \text{if $i\in sol$} \\
            \min\{w_{i}, ..., w_n\} & \text{if $i\notin sol$}
        \end{array}\right.
    \end{align*}
    to be the value of the rest of the items, weight of the rest of the items, and minimum weight for each of the rest of the items.
    
    \item Encoding of the previous action, with the action types and encoding sizes listed in \tabref{tab:actions-and-encoding-size-knapsack}.
\end{itemize}

\begin{table}[h]
    \centering
    \caption{Action types and action argument encoding sizes for the 0/1 knapsack problem.}
    \label{tab:actions-and-encoding-size-knapsack}
    \begin{tabular}{c|c}
    \toprule
        Action Type & Encoding Size \\
        \midrule
        Put() & 0 \\
        Pop() & 0 \\
        MoveVar(+1/-1) & 1 \\
        Knapsack() & 0 \\
        Return() & 0 \\
        \bottomrule
    \end{tabular}
\end{table}

The total size of the input $s_t$ is therefore $7$ (custom features) $+5$ (action type one-hot) $+1$ (action argument encoding) $+1$ (whether previous action is None)$=14$.

Note that, like most other tasks, the provided input interface is sufficient for implementing the motivating algorithm, in this case the DFS algorithm, but also contains more features such that the model is free to use them if they are helpful, enabling them to learn different and potentially better algorithms.

For this task the features that involve $v_\text{rest}, w_\text{rest}, w_\text{min}$ require linear time to compute.  However since the knapsack problem is NP-complete, anything less than exponential time is acceptable.

\section{Program search space size estimation}
\label{appendix:search-space-size}

Each program in our setup corresponds to a mapping $f$ that maps a state $s$ to an output instruction $a$. The number of unique programs in our search space is therefore $|\mathcal{A}|^{|\mathcal{S}|}$, where $|\mathcal{A}|$ is the number of possible instructions in the output space, and $|\mathcal{S}|$ is the size of the state space.  A program needs to specify what instruction to output for each possible state, and two programs are different even if they differ on just one state.

\textbf{Bubble / insertion interface. } Here we try to estimate the search space size for programs in the bubble / insertion sort interface.  In this case we have $k$ different SwapWithNext($i$) instructions, $2k$ different MoveVar($i$, -1/+1) instructions, and $k^2$ different AssignVar($i, j$) instructions.  Therefore the total number of unique instructions is $|\mathcal{A}|=k+2k+k^2=28$ as $k=4$ in our case.

As calculated in \appendixref{appendix:bubble-insertion-interface-details}, the input is represented by a 68-dimensional binary vector.  A naive estimate of the state space size would be $|\mathcal{S}|=2^{68}\approx 2.95\times 10^{20}$.  However, since the bits in this 68-dimensional vector are not all independent, and the actual state space size can be estimated more accurately.

In particular, the first part of $s_t$ contains the comparison results for each pair of $1\le i, j \le k, i\ne j$, involving ${k\choose 2} = k(k-1)/2$ pairs, and for each pair we compare $i$ with $j$ but also $\A[i]$ with $\A[j]$, each with 3 independent possibilities, this adds up to $(3\times 3)^{k(k-1)/2}$ different possibilities.

The next part of $s_t$ contains information about how $\A[v_i]$ compares with $\A[v_i-1]$ and $\A[v_i+1]$ for each $i$, each with $4\times4$ possibilities, which adds up to $(4\times 4)^k$ possibilities total.

Putting these two parts together, we have a better estimate of the state space size as
\begin{equation*}
    |\mathcal{S}|=(3\times 3)^{k(k-1)/2} \times (4\times 4)^k = 3.48\times 10^{10}.
\end{equation*}

This leads to a search space size of $28^{3.48\times 10^{10}}$.  This number (multiplying $3.48\times 10^{10}$ copies of the number $28$ together) is so huge that a naive enumeration of the search space is infeasible.  This also shows how challenging finding a correct program in this space is.

\textbf{Quick sort interface. } With the quick sort interface, the search space for programs is even larger.  First of all, the size of the input $s_t$ grows from 68-dimensional to 129-dimensional.  Second of all, the output space size grows from 28 to $k+2k+k^2+2k^5+k+k^2=2096$.  The biggest factor of $2k^5$ comes from the function call instruction, with $id\in\{1,2\}$ and 5 variable IDs each with $k$ possibilities.

A naive calculation of the search space leads to an estimate of $2096^{2^{129}}$.  However we can refine the estimate of the state space size similar to what we did with the bubble / insertion interface.  But it is easy to see this search space is much larger than the bubble / insertion search space, hence the problem is much harder to solve.

\section{Theoretical performance limit of the bubble / insertion interface}
\label{appendix:limit-of-bubble-insertion-interface}

In this section we show that the optimal performance achievable with the bubble / insertion interface is $O(n^2)$, making it clear that even though all the 3 interfaces we discussed for sorting can achieve correctness, the specific interface and instruction set still plays a significant role in efficiency leading to $O(n^3), O(n^2)$ and $O(n\log n)$ algorithms respectively.

The theoretical result states the complexity of an algorithm class for a particular type of input data distributions.  In our case the results are w.r.t. the uniformly permuted lists, \ie, start with a list $\A=[0, 1, ..., n-1]$ and then uniformly perturb it, such that every possible permutation of $\A$ has the same probability.  This can be implemented, for example, by iterating over index $i$ from $0$ to $n-1$ and swapping $\A[i]$ with $\A[j]$ where $j$ is uniformly sampled from the range $[i, i+1, ..., n-1]$. Note the actual numeric values in $\A$ doesn't matter for our complexity results.  If these values are arbitrary, we assume a way to break ties if there are duplicated values, such that the elements in $\A$ can always be ordered unambiguously.

Under this uniform distribution, we can easily show the following:
\begin{equation}
    p(\A[i] < \A[j]) = p(\A[i] > \A[j]) = \frac{1}{2}, \qquad \forall i, j\qquad 0\le i < j < n.
\end{equation}
This is true because this uniform distribution assigns equal probability for each permutation, and we have equal number of permutations where $\A[i] < \A[j]$ vs. $\A[i] > \A[j]$.

We can now state the following theorem:

\begin{theorem}
\label{thm:bubble-insertion}
For uniformly permuted lists of size $n$, an agent restricted to only use SwapWithNext instruction to manipulate data needs on average $\Theta(n^2)$ SwapWithNext instructions to sort the list correctly.
\end{theorem}

\begin{proof} We define an inversion as a pair of $(i, j)$ such that $i < j$ but $\A[i] > \A[j]$.  For uniformly permuted lists, we can show that the expected number of inversions is
\begin{equation}
    \expt\left[\sum_{0\le i < j < n} \I[\A[i] > \A[j]]\right] = \sum_{0\le i < j < n} \expt[\I[\A[i] > \A[j]]] = \sum_{0\le i<j<n} \frac{1}{2} = \frac{n(n-1)}{4},
\end{equation}
where the expectation is taken over the data distribution, $\I[.]$ is an indicator function, so that $\expt[\I[\A[i]>\A[j]]] = p(\A[i] > \A[j])=1/2$.

Since each SwapWithNext intruction only swaps two neighboring elements, it can reduce the number of inversions at most by 1.  This immediately implies that we need in expectation $\frac{n(n-1)}{4}=\Theta(n^2)$ SwapWithNext instructions to reduce the number of inversions to 0.
\end{proof}

\textbf{Remark: } The SwapWithNext instruction as described in \secref{sec:bubble-insertion} swaps $\A[i]$ with $\A[i+1]$ at a position $i$ (in \secref{sec:bubble-insertion} this is represented through a variable $v_i$), therefore such an instruction only swaps neighboring elements in the list.  The theorem applies to any agents / algorithms that changes the data in the list only through the SwapWithNext instruction, and does not restrict other types of instructions that do not change the data in the list, for example, the MoveVar and AssignVar instructions in our bubble / insertion interface.  The theorem states that as long as the agent complies with this restriction, no matter what the input observation is, no matter what extra instruction types we have (that does not change the data in the list), we need at least $\Theta(n^2)$ SwapWithNext instructions to sort a uniformly permuted list of size $n$ on average, and potentially more if counting the other types of instructions, therefore the optimal complexity for algorithms in this class is $O(n^2)$ at best.

The fact that this theorem holds regardless of what kind of input observations the agent can see also shows the importance of the instruction set.

\section{Model details}
\label{appendix:model-details}

\subsection{Network architectures}

\subsubsection{Graph Neural Networks (GNNs)}
We use GNNs to handle the graph structured inputs.  Denote the input node feature for node $v\in V$ as $\xv_v$, and edge feature for edge $(u,v)\in E$ as $\xv_{uv}$, then the GNN computes representations for each node through the following message passing process:
\begin{align}
    \hv_v^{(0)} &= \MLP_\text{embed}(\xv_v) \\
    \mv^{(t)}_{u\rightarrow v} &= \MLP_\text{edge}^{(t)}\left(\left[\hv_v^{(t)}, \hv_u^{(t)}, \xv_{uv}\right]\right) \\
    \hv_v^{(t+1)} &= \MLP_\text{node}^{(t)}\left(\left[\hv_v^{(t)}, \frac{1}{|\mathcal{N}(v)|}\sum_{u\in \mathcal{N}(v)} \mv^{(t)}_{u\rightarrow v}\right]\right),
\end{align}
where $\MLP_\text{embed}, \MLP_\text{edge}^{(t)}, \MLP_\text{node}^{(t)}$ are individual MLPs, $[.]$ represent vector concatenation and $\mathcal{N}(v)=\{u|(u,v)\in E\}$ represents the set of incoming neighbors of $v$.

For the full view interface for sorting (\secref{sec:model-full-view}), we used a 5-layer GNN with node state dimension 16 (size of $\hv_v$).  Each of the individual MLPs has 1 hidden layer, with the following layer sizes: $\MLP_\text{embed}$ - [32, 16], $\MLP_\text{edge}^{(t)}$ - [32, 32], $\MLP_\text{node}^{(t)}$ - [32, 16].  Here the first number is the size of the hidden layer, and the second number is the size of the output of the MLP.

The full view interface requires the model to predict two indices.  We use a mechanism similar to the pointer network \cite{vinyals2015pointer} to select nodes from the graph.  More specifically, once we get all the node representations $\hv_v^{(T)}$, we apply another MLP to predict a logit value for each node, then the distribution over indices is a softmax over those logits.  The second index is predicted auto-regressively, and the conditioning on the first index is encoded by appending one extra bit to each node feature vector, indicating which node was selected as the first index.

\subsubsection{MLPs}

For the other tasks we have explored, we use a vectorized representation for $s_t$, and therefore further use an MLP to get representations of $s_t$.  For all the MLPs across all tasks, we use a network with 3 hidden layers with size 64 each.

For outputing actions, we use one MLP for predicting the action type, and one MLP for each action argument, with appropriate output distributions for each type.  More specifically, we use Bernoulli distributions for boolean arguments, and categorical (softmax) distributions for integer arguments.  We do not use the pointer type in this setting.  For each output argument, the state representation is passed through the MLP and then projected to the appropriate output dimension with a linear layer and then apply softmax or sigmoid to get the appropriate probability values.

To make the arguments auto-regressive, we augment the input vector with the encoding of the previously selected action argument, binary for boolean arguments and one-hot for integer arguments and feed the augmented input vector through the MLP to make predictions.

\subsection{Shaping reward}
\label{appendix:shaping-reward}

For the sorting task, in addition to the standard per-step penalty $r_t=-c$, we also explored a form of incremental shaping reward, defined in terms of the orderedness of the array
\begin{equation}
    h(\A) = \sum_{i=\low}^{\high-1} \I[\A[i] \le \A[i+1]],
\end{equation}
where $\I[.]$ is the indicator function.  This quantity counts the number of neighboring pairs in the correct order. We define the incremental reward as
\begin{equation}
    r_t = h(\A_{t+1}) - h(\A_t) - c
\end{equation}
where $h(\A_t)$ is the orderedness of the list $\A$ at time step $t$ and $h(\A_{t+1})$ is the orderedness at the next time step, after executing the action $a_t$.  Note that since each swap action only makes local changes to the array, the difference in orderedness after an action can be computed in constant time.

In our experiments, we found that using this shaping reward can help the learning take off much faster, but the model is also more easily trapped in local optima.  In particular, all of our models that outperform quick sort are not trained with this shaping reward.

Also note that this is just one form of shaping reward, and there are plenty more options for each domain that we haven't fully explored.  Carefully designed reward may play a big role in reinforcement learning.

\subsection{Hyperparameters for learning}

For each task, we run a sweep over hyperparameters for each setting.  The hyperparameters swept over include:
\begin{itemize}
    \item Learning rate $\in\{10^{-4}, 10^{-5}\}$.
    \item Discount $\gamma\in\{0.9, 0.99\}$ for full-view sorting, searching in sorted lists, and knapsack, and $\gamma\in\{0.99, 0.999\}$ for sorting with bubble / insertion and quicksort interfaces.
    \item Weight for the entropy loss $\in\{0, 10^{-3}\}$.
    \item Number of steps $n$ in $n$-step policy gradient: $n\in\{80, 160, 320\}$ for sorting with bubble / insertion or quick sort interface, and $n=50$ for full-view sorting, searching in sorted lists and knapsack.
    \item Weight for the baseline loss $\mu=10^{-3}$.
\end{itemize}

\section{Function call semantics}
\label{appendix:func-call}

Here we give a more detailed description about the semantics of the function call instructions introduced in \secref{sec:functions}.

We introduced two new types of instructions:
\begin{gather*}
    \text{FunctionCall}(id, l_1, ..., l_p, o_1, ..., o_p, r_1, ..., r_q) \\
    \text{Return}(l'_1, ..., l'_q),
\end{gather*}
where $id$ is an integer valued function ID.
Each $l_i$ is an ID of the local (inside the function) variable to be assigned, $o_i$ is an ID of the outer-scope (out of the function, when calling the function) variable to be passed in, $r_i$ is an ID of the outer scope variable to receive the return value, and  $l'_i$ is an ID of the local variable whose value will be returned. $p$ is the number of input arguments this function accepts, and $q$ is the number of return values this function has.  In our sorting example, $p=2$, $q=1$, and each $l_i, o_i, r_i, l'_i\in\{1, ..., k\}$ is a variable ID.

To support these two instructions, the environment maintains a function call stack, and each stack entry keeps track of the necessary information useful for recovering the execution once returned from the current function.
Each stack entry contains the variable values before the function call, as well as the ID of the variables to receive the return values, plus any other state information that need to be kept local.

The environment also keeps track of the previous action being executed.  When entering a new function, the previous action is set to None.  The function call stack entry also keeps track of the previous action before the function call to properly resume execution after return.

\begin{figure}
    \centering
    \includegraphics[width=1.0\columnwidth]{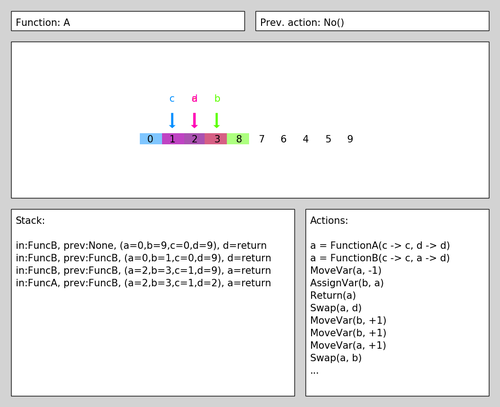}
    \caption{One step in an example episode for sorting with the quick sort interface.}
    \label{fig:quick-sort-interface}
\end{figure}

\figref{fig:quick-sort-interface} shows one step in one example episode for sorting with the quick sort interface that uses function calls, which may help better understand how functions work in our setup.  This visualization shows the array $\A$ in the center, the $k=4$ index variables listed as $a, b, c, d$ arrows (with overlay), the current function ID (top-left corner), the previous action (top-right corner), the call stack (bottom-left) with the most recent entry at the bottom, and the action trace (bottom-right) with the most recent action at the top of the trace.

Note that we replaced the variable IDs $1\le i\le k$ with letters $a,b,c,d$ as they are more intuitive to us to keep track of, essentially the variable IDs can be treated as variable names.  Similarly, we replaced function IDs with letters $A, B, ...$ as they are easier to understand and function IDs are equivalent to function names.

This visualization corresponds to the step right after a function call with arguments $id=A, l_1=c, l_2=d, o_1=c, o_2=d, r_1=a$, which translates into ``call function A, and set the local variables inside the function $c$ and $d$ with values of $c$ and $d$ in the current scope (the variable values in this level that calls the function), and the return value from the function should be assigned to variable $a$''.

The environment receives this function call and pushes one entry onto the stack, in this case the bottom-most entry in the stack visualization.  Each entry contains:
\begin{itemize}
    \item The function ID of the outer scope at the calling level, in this case function B (``prev:FuncB'').
    \item The previous action at the calling level, in this case calling function A (indicated by ``in:FuncA'' which also indicates we are currently in function A).  Note that this is the ``previous action'' after we return from the function call, so this is always the function call itself.
    \item The variable values before entering the function $(a=2, b=3, c=1, d=2)$.
    \item The variable receiving the return value, in this case $a$.
\end{itemize}

When entering a function, the environment also resets the function ID (top-left corner) to the requested one, and resets the previous action to None (top-right corner).

The agent always predicts the next action based on its input state $s_t$, so it can only notice entering a new function when the function ID changed and / or the previous action became None.

The Return($l'$) action does the opposite.  Once received such an action, the environment pops one entry from the stack, and then resets the function ID, the previous action, the variables values to the recorded values, and lastly, assign the return value specified by variable $l'$ to and overwrite the variable recorded in the stack entry.

We include two more videos showing the execution traces of the sorting agents to further illustrate how our interfaces and functions work.

\section{Implementing known algorithms using our interfaces}\label{appendix:scripted-sorting-algorithms}

In this section we show how we converted known algorithms into our framework as scripted agents, such that they observe the same input state, and emit actions in the same action space as our models.  This exercise can help us identify what are the necessary instructions and input data representation for our models to learn the right algorithms.  These scripted agents can also act as teachers in our imitation learning setting.

\begin{algorithm}[h]
\caption{Bubble sort}\label{alg:bubblesort}
\begin{algorithmic}[1]
\Procedure{BubbleSort}{\A, \low, \high}
\For{$j\gets\high \text{ to } \low$}
\For{$i\gets\low \text{ to } j-1$}
    \If{$\A[i] > \A[i+1]$}
        \State Swap $\A[i], \A[i+1]$
    \EndIf
\EndFor
\EndFor
\EndProcedure
\end{algorithmic}
\end{algorithm}

\algref{alg:bubblesort} above is one example implementation of the classic bubble sort algorithm, it is clear from the double-loop that this algorithm runs in $O(n^2)$ time.  The converted scripted agent is shown in \algref{alg:bubblesort-agent}.

\begin{algorithm}[h]
\caption{Bubble sort agent}\label{alg:bubblesort-agent}
\begin{algorithmic}[1]
\Procedure{BubbleSortAgent}{input state}
\State Let $i=1, j=2, l=3$
\If{$v_i < v_j$}
    \If{$\A[v_i] > \A[v_i + 1]$}
        \returnstmt SwapWithNext($i$)
    \Else
        \returnstmt MoveVar($i$, +1)
    \EndIf
\ElsIf{$v_i = v_j$}
    \State MoveVar($j$, -1)
\Else
    \State AssignVar($i, l$)
\EndIf
\EndProcedure
\end{algorithmic}
\end{algorithm}

Note that to be consistent with our action space, we instead used variables $v_i$ and $v_j$, and used $i,j$ as aliases to index these variables.  Also note that here we used $l$ to refer to the index of the variable that has $\low$ as its value.  Because of the way we initialize the variables, $v_1=v_3=\low$ and $v_2=v_4=\high$, it is easy to use one variable to track the value of $\low$.

Also note that this scripted agent does not terminate, and this is true for all our setups, as termination is handled in the environment.  For sorting, the environment keeps track of the number of neighboring pairs in the right order, and terminates execution when this number reaches $n$, updating this number is a constant time operation.

Similarly, \algref{alg:insertionsort} and \algref{alg:insertionsort-agent} shows an example implementation of insertion sort and how it translates into a scripted agent in our framework.

\begin{algorithm}[h]
\caption{Insertion sort}\label{alg:insertionsort}
\begin{algorithmic}[1]
\Procedure{InsertionSort}{\A, \low, \high}
\For{$i\gets\low \text{ to } \high$}
\State $j\gets i$
\While{$j > \low\text{ and }\A[j] < \A[j-1]$}
    \State Swap $\A[j], \A[j-1]$
    \State $j\gets j-1$
\EndWhile
\EndFor
\EndProcedure
\end{algorithmic}
\end{algorithm}

\begin{algorithm}[h]
\caption{Insertion sort agent}\label{alg:insertionsort-agent}
\begin{algorithmic}[1]
\Procedure{InsertionSortAgent}{input state}
\State Let $i=1, j=2$
\If{$v_i < v_j$}\Comment{Sets initial value of $v_j$}
\returnstmt AssignVar($j, i$)
\ElsIf{$v_i=v_j$}
\returnstmt MoveVar($i$, +1)
\Else
    \If{$\A[v_j] > \A[v_j+1]$}
        \returnstmt SwapWithNext($j$)
    \ElsIf{$v_j > \low \text{ and }\A[v_j] < \A[v_j-1]$}
        \returnstmt MoveVar($j$, -1)
    \Else
        \returnstmt AssignVar($j, i$)
    \EndIf
\EndIf
\EndProcedure
\end{algorithmic}
\end{algorithm}

\algref{alg:quicksort} shows an example implementation of quick sort, which always picks the high end of the range as the pivot value for partitioning the list.  This implementation uses two functions, QuickSort which sorts a range of an array, and Partition which partitions an array range into two parts, with regard to a pivot value, the elements in the range smaller than the pivot is put on the left side of the pivot, and the elements larger than the pivot is put on the right side of it.  Notably, the QuickSort function calls itself recursively.

\begin{algorithm}[h]
\caption{Quick sort}\label{alg:quicksort}
\begin{algorithmic}[1]
\Procedure{QuickSort}{\A, \low, \high}
\If{$\low < \high$}
    \State $i\gets$ Partition($\A, \low, \high$)    \label{line:block1-start}
    \State QuickSort($\A, \low, i-1$)
    \State QuickSort($\A, i+1, \high$)  \label{line:block1-end}
\EndIf
\EndProcedure
\State
\Procedure{Partition}{\A, \low, \high}
\State $i\gets\low$
\For{$j\gets \low\text{ to }\high-1$}
    \If{$\A[j] < \A[\high]$}\Comment{$\A[\high]$ as pivot}
        \State Swap $\A[i], \A[j]$  \label{line:block2-start}
        \State $i\gets i + 1$   \label{line:block2-end}
    \EndIf
\EndFor
\State Swap $\A[i], \A[\high]$
\returnstmt $i$
\EndProcedure
\end{algorithmic}
\end{algorithm}

This quicksort implementation has one subtlety that is less of a problem for the bubble / insertion sort implementations.  Here we have in some cases more than one statements in one code block.  For example the code block from line \ref{line:block1-start} to \ref{line:block1-end} is under the same if-branch, in order to correctly predict which action to use, our model must have enough knowledge about the execution state to disambiguate the conditions for the 3 different situations.

We found adding the previously executed action to the input state $s_t$ is sufficient to handle disambiguation for this quick sort implementation.  \algref{alg:quicksort-agent} shows the converted quick sort scripted agent.

In this implementation, we used ``prev'' to denote the previous action, and used the aliases $i=1, j=2, l=3, h=4$ to index the 4 variables $v_1, v_2, v_3, v_4$, such that $v_i$ and $v_j$ are two free variables that roughly corresponds to the $i$ and $j$ variables in the Partition function, $v_l$ and $v_h$ are the two variables that corresponds to the $\low$ and $\high$ variables in \algref{alg:quicksort}.  We also used the more intuitive notation ``$v_a\gets$Function$f$($v_b\gets v_c, v_d\gets v_e$)'' to represent the FunctionCall$(f, b, d, c, e, a)$ which means call function $f$, and assign the current value of $v_c$ and $v_e$ to the local variables $v_b$ and $v_d$ in the function, and assign the return value to $v_a$.  In this implementation, function 1 corresponds to the QuickSort function in \algref{alg:quicksort} and function 2 corresponds to the Partition function.  As we can see from \algref{alg:quicksort}, the function QuickSort does not return, however to be compatible with the function call action interface, our function 1 still returns a variable $v_h$, however it is never used.  On the other hand, function 2 returns the pivot index $v_i$.

\algref{alg:binary-search-agent} shows the scripted agent for binary search in ordered lists.  Note that here we only used 3 variables and again the use of the previous action helps disambiguate different situations.

\begin{algorithm}[h]
\caption{Binary search agent}\label{alg:binary-search-agent}
\begin{algorithmic}[1]
\Procedure{BinarySearchAgent}{input state}
\State Let $i=1, l=2, h=3$
\If{$v_l > v_h$ or ($v_i = v_l$ and $\A[v_i] > q$) or ($v_i = v_h$ and $\A[v_i] < q$)}
    \returnstmt NotFound()
\ElsIf{$\A[v_i] = q$}
    \returnstmt Found($i$)
\ElsIf{prev in (None, AssignVar($l, i$), Assign($h, i$))}
    \returnstmt AssignMid($i, l, h$)
\ElsIf{prev = AssignMid($i, l, h$)}
    \If{$\A[v_i] < q$}
        \returnstmt MoveVar($i$, +1)
    \Else
        \returnstmt MoveVar($i$, -1)
    \EndIf
\ElsIf{prev = MoveVar($i$, +1)}
    \returnstmt AssignVar($l, i$)
\ElsIf{prev = MoveVar($i$, -1)}
    \returnstmt AssignVar($h, i$)
\Else
    \returnstmt None
\EndIf
\EndProcedure
\end{algorithmic}
\end{algorithm}

\algref{alg:dfs-knapsack-agent} shows the scripted DFS agent for the 0/1 knapsack problem.  Note that here ``Knapsack()'' is a function call, since we only need one function in this setup, the function ID is ignored, and since this function does not need arguments or return values, this action doesn't take any arguments.  $i$ is the current index value, and $w$ is the weight of the current solution so far.

\begin{algorithm}[h]
\caption{DFS knapsack agent}\label{alg:dfs-knapsack-agent}
\begin{algorithmic}[1]
\Procedure{DfsKnapsackAgent}{input state}
\If{prev = None}
    \If{$i \ge n$ or $w > W$}
        \returnstmt Return()
    \Else
        \returnstmt Put()
    \EndIf
\Else
    \If{prev = Put()}
        \returnstmt MoveVar(+1)
    \ElsIf{prev = MoveVar(+1)}
        \returnstmt Knapsack()
    \ElsIf{prev = Knapsack()}
        \returnstmt MoveVar(-1)
    \ElsIf{prev = MoveVar(-1)}
        \If{$i\in sol$}
            \returnstmt Pop()
        \Else
            \returnstmt Return()
        \EndIf
    \ElsIf{prev = Pop()}
        \returnstmt MoveVar(+1)
    \Else
        \returnstmt None
    \EndIf
\EndIf
\EndProcedure
\end{algorithmic}
\end{algorithm}

\paragraph{Remark: } We can see from the example implementations of the scripted agents in this section that the design decisions of our interfaces, \ie~the inputs and output actions, are heavily inspired by the known algorithms.  On the other hand, these interfaces are designed to contain more information than needed by these algorithms, such that our models have the opportunity to learn and surpass the known algorithms.

\begin{algorithm*}[ht]
\caption{Quick sort agent}\label{alg:quicksort-agent}
\begin{algorithmic}[1]
\Procedure{QuickSortAgent}{input state}
\State Let $i=1, j=2, l=3, h=4$
\If{FunctionID = None}
    \returnstmt $v_h\gets$ Function1($v_l \gets v_l, v_h\gets v_h$)
\ElsIf{FunctionID = 1}\Comment{QuickSort}
    \If{$v_l < v_h$}
        \If{prev = None}
            \returnstmt $v_i\gets$ Function2($v_l\gets v_l, v_h\gets v_h$)
        \ElsIf{prev = ($v_i\gets$ Function2($v_l\gets v_l, v_h\gets v_h$))}
            \returnstmt AssignVar($j, i$)
        \ElsIf{prev = AssignVar($j, i$)}
            \returnstmt MoveVar($i$, -1)
        \ElsIf{prev = MoveVar($i$, -1)}
            \If{$v_i > v_l$}
                \returnstmt $v_i\gets$ Function1($v_l\gets v_l, v_h\gets v_i$)
            \Else
                \returnstmt MoveVar($j$, +1)
            \EndIf
        \ElsIf{prev = ($v_i\gets$ Function1($v_l\gets v_l, v_h\gets v_i$))}
            \returnstmt MoveVar($j$, +1)
        \ElsIf{prev = MoveVar($j$, +1) and $v_j < v_h$}
            \returnstmt $v_h\gets$ Function1($v_l\gets v_j, v_h\gets v_h$)
        \Else
            \returnstmt Return($h$)
        \EndIf
    \Else
        \returnstmt Return($h$)
    \EndIf
\Else\Comment{Function ID = 2, Partition}
    \If{prev = None}
        \returnstmt AssignVar($i, l$)
    \ElsIf{prev = AssignVar($i, l$)}
        \returnstmt AssignVar($j, l$)
    \ElsIf{$v_j < v_h$}
        \If{prev = Swap($i, j$)}
            \returnstmt MoveVar($i$, +1)
        \ElsIf{(prev = AssignVar($j, l$) or prev = MoveVar($j$, +1)) and $\A[v_j] < \A[v_h]$}
            \If{$v_i \ne v_j$}
                \returnstmt Swap($i, j$)
            \Else
                \returnstmt MoveVar($i$, +1)
            \EndIf
        \Else
            \returnstmt MoveVar($j$, +1)
        \EndIf
    \ElsIf{prev = MoveVar($j$, +1)}
        \returnstmt Swap($i, h$)
    \Else
        \returnstmt Return($i$)
    \EndIf
    
\EndIf
\EndProcedure
\end{algorithmic}
\end{algorithm*}

\section{More experiment results}
\label{appendix:more-results}

In this section we present a few additional experiment results to supplement the results in the main paper.

\begin{table*}[th]
    \centering
    \caption{Performance for models that uses the full-view interface.  Each cell shows the average episode length / percentage of episodes solved.  The maximum episode length is $n^2$.  The cells {\setlength{\fboxsep}{0pt}\colorbox{gray!50}{high-lighted in gray}} indicates the training size range.}
    \label{tab:full-view-efficiency}
    \begin{tabular}{rcccccc}
    \toprule
        Instance size & 5 & 10 & 20 & 30 & 40 & 50 \\
    \midrule
        Teacher & 2.8 / 100 & \graycell 6.8 / 100 & \graycell 16.6 / 100 & 25.9 / 100 & 35.4 / 100 & 45.7 / 100 \\
        Supervised learning & 8.8 / 72 & \graycell 6.8 / 100 & \graycell 16.6 / 100 & 611.4 / 33 & 1600.0 / 0 & 2500.0 / 0\\
        RL & 2.8 / 100 & \graycell 6.8 / 100  & \graycell 16.6 / 100 & 26.0 / 100 & 1506.7 / 6 & 2500.0 / 0 \\
        \bottomrule
    \end{tabular}
\end{table*}

\tabref{tab:full-view-efficiency} shows the efficiency results for the agents trained with the full-view interface.  With the full-view observation, our models can learn to fit data in the training range quite well, achieving 100\% solve rate and solving all the instances using exactly the same number of steps as the teacher.  However, beyond the training range, the solve rate drops quickly, and the episode length also increases rapidly.  We have also tried to use RL in this setting, notably, with RL the learned model seems a bit more robust, and generalizing slighly better than the one trained with pure supervised learning.

\begin{table}[th]
    \centering
    \caption{Evaluating quick sort and learned model on large instances.}
    \label{tab:eval-quick-sort-large}
    \begin{tabular}{rcc}
    \toprule
        Instance Size & 10,000 & 100,000 \\
        \midrule
        Quick sort & 368,338.6 & 4,602,545.1 \\
        RL + imitation & \bf 338,947.9 & \bf 4,227,102.4 \\
        \bottomrule
    \end{tabular}
\end{table}

\tabref{tab:eval-quick-sort-large} shows the evaluation results for quick sort and the learned model on very large instances.  The evaluation of the learned model on 100 instances of size 100,000 took 3 days.  This evaluation is slow as we need to evaluate the neural network millions of times sequentially.  None of the scripted or learned $O(n^2)$ algorithms finish evaluation in 3 days.

Note that even though the model was trained on instances of size 30-50, they still generalize perfectly to instances of size 100,000, more than 1000 times larger than the instances seen by the model during training.  Also, these large instances require millions of steps to solve, also thousands of times more than required during training.

\figref{fig:curves-bubble-insertion}, \ref{fig:curves-quick-sort-reward}, \ref{fig:curves-quick-sort-unroll-len} show a few typical training curves for the bubble-insertion interface and the quick-sort interface, plotting the average episode length evaluated on 100 validation instances.  We also show the performance for the corresponding teacher agent, insertion sort and quick sort respectively in all plots for comparison.  For the bubble-insertion interface (\secref{sec:bubble-insertion}), we train on instances of size 10-20 and during training, the maximum episode length is set to $20^2=400$, and therefore all the curves start from 400.  For any instance, an episode length smaller than 400 therefore indicates that the agent successfully solved it, as an episode terminates when the instance is correctly solved, or when reaching the maximum episode length.  We can observe that the model learned to do sorting quite quickly, and spend the most of the training process refining the policy.  In this case the models were trained with shaping reward and an imitation loss.

We show both the average across 5 seeds (left figure) as well as the best across 5 seeds (right figure).  For this task we care mostly about the best performance, as we are trying to find the best neural program for the task, and it doesn't matter which run the best model comes from.  As you can see from \figref{fig:curves-bubble-insertion} (b) our learning process can find algorithms that outperform the teacher quite quickly.

\begin{figure}[t]
    \centering
    \begin{tabular}{cc}
        \includegraphics[width=0.45\textwidth]{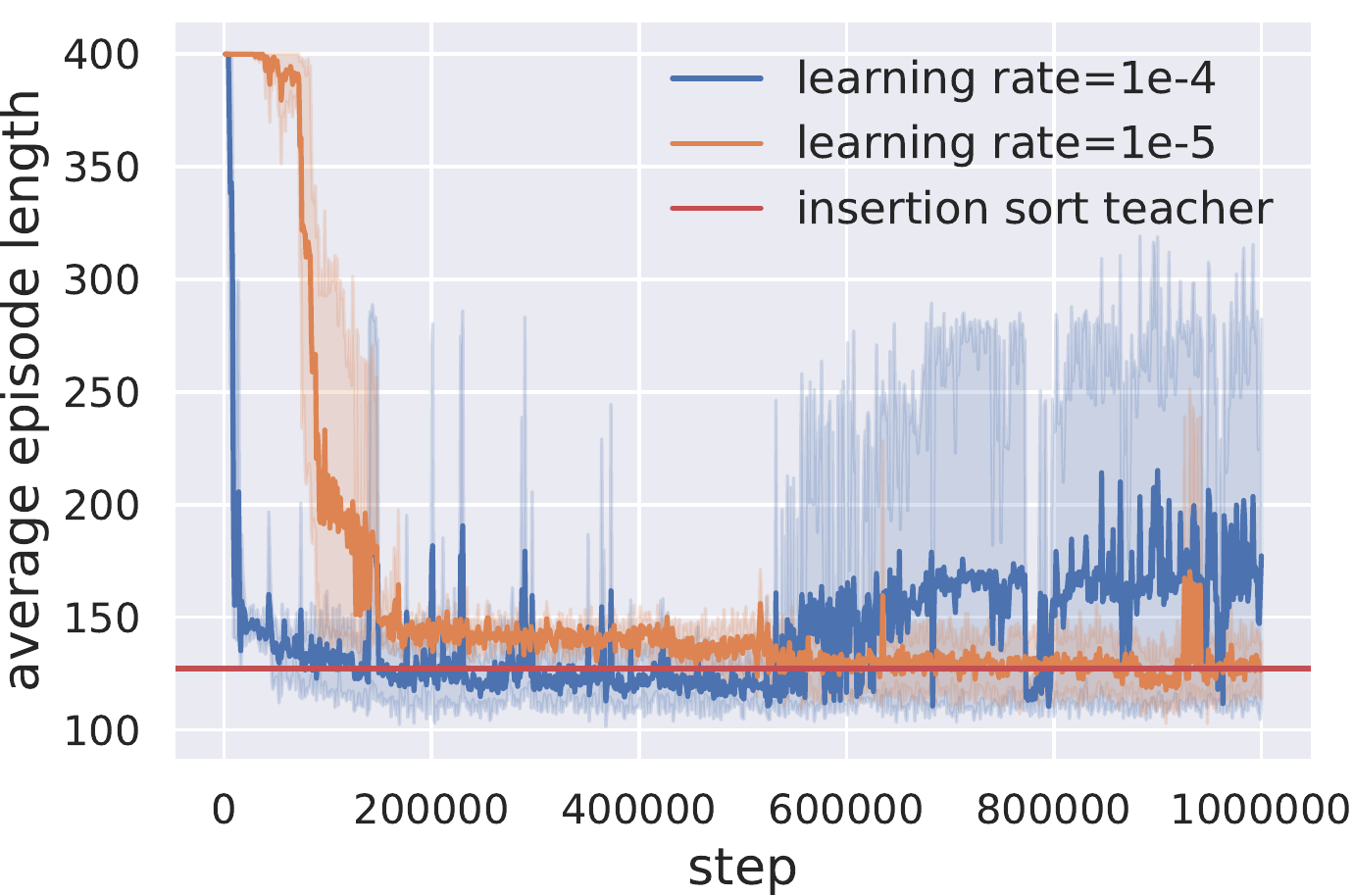} & 
        \includegraphics[width=0.45\textwidth]{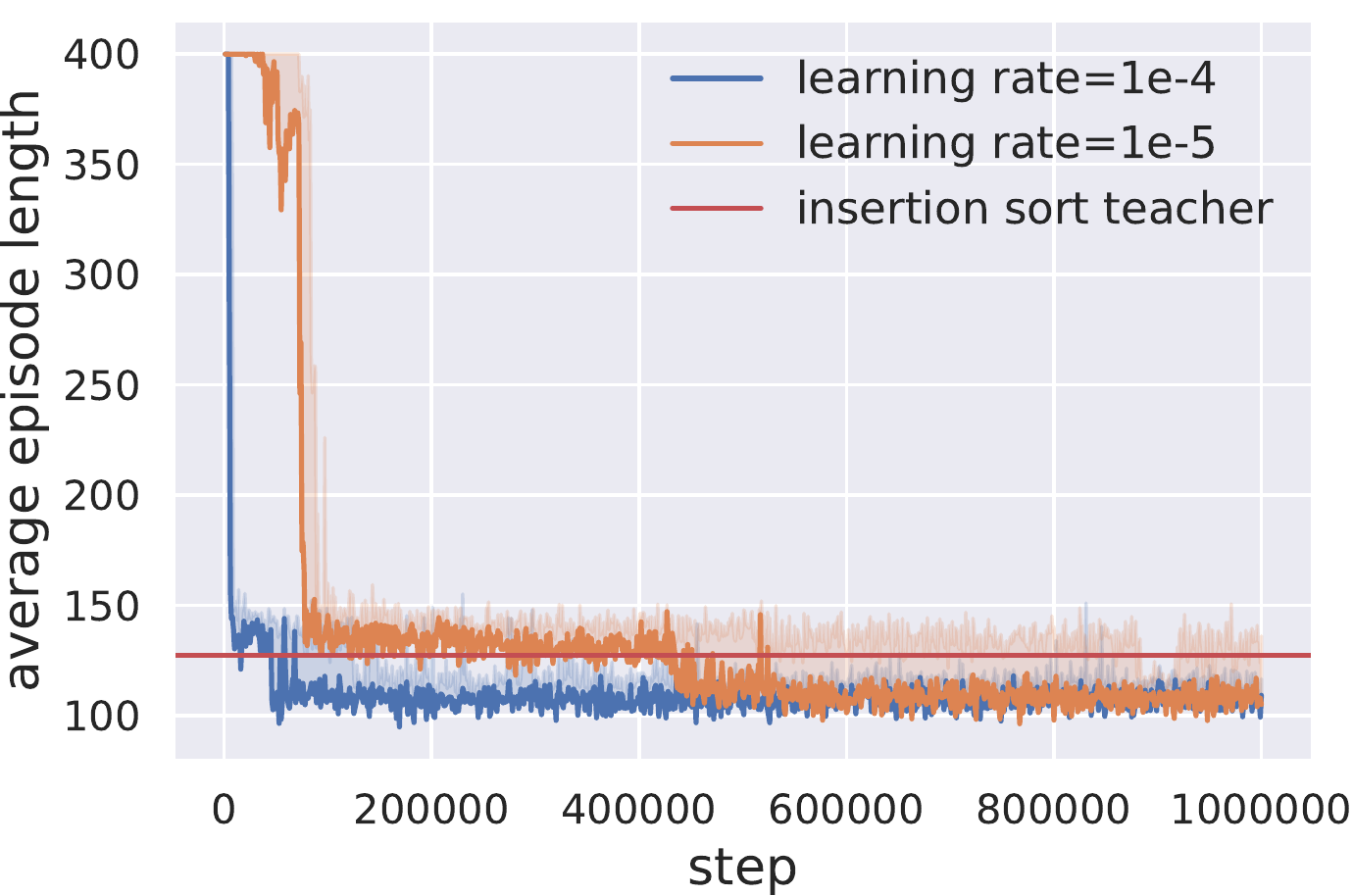}
        \\
        (a) Average across 5 seeds & (b) Best across 5 seeds
    \end{tabular}
    \caption{Typical training curves for the bubble-insertion interface (\secref{sec:bubble-insertion}).  The curves show the mean (left) and minimum (right) of the average episode length metric for models trained across 5 random seeds for the same hyperparameter setting, as well as a 95\% confidence interval.  This figure compares the effect of different learning rate values on the performance.}
    \label{fig:curves-bubble-insertion}
\end{figure}

\begin{figure}[t]
    \centering
    \begin{tabular}{cc}
        \includegraphics[width=0.45\textwidth]{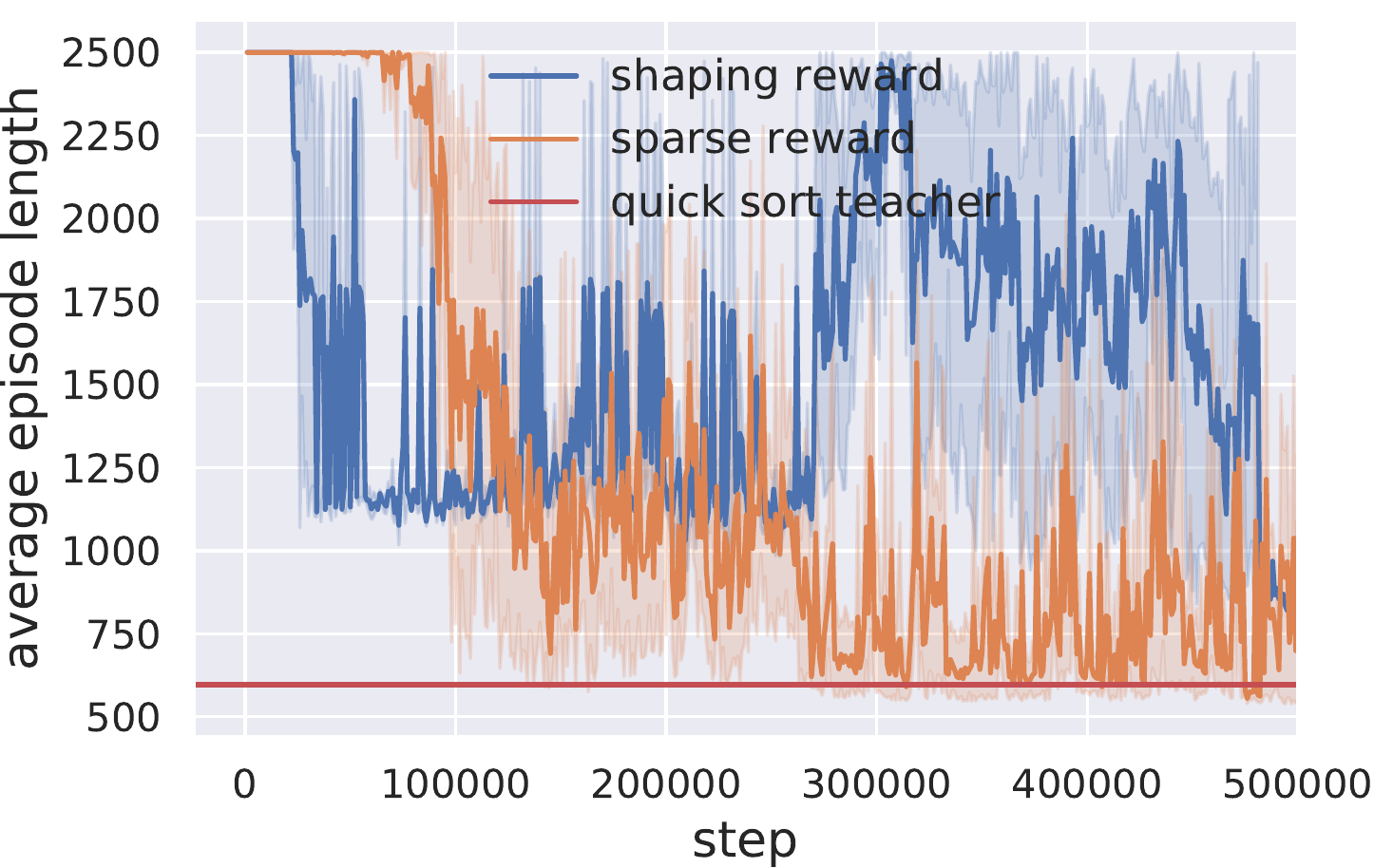} & 
        \includegraphics[width=0.45\textwidth]{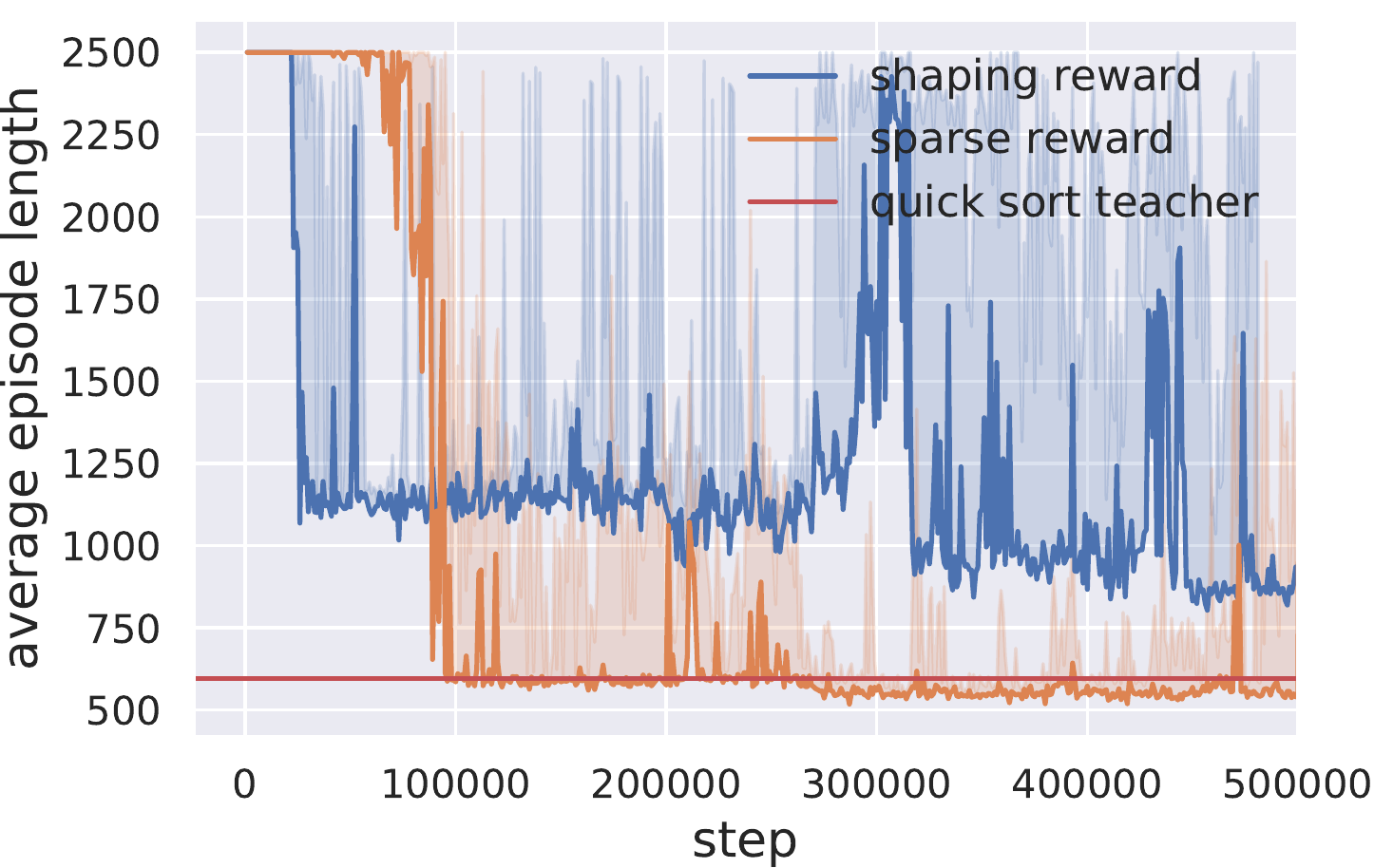}
        \\
        (a) Average across 5 seeds & (b) Best across 5 seeds
    \end{tabular}
    \caption{Typical training curves for the quick sort interface (\secref{sec:functions}). The curves show the mean(left) and minimum (right) of the average episode length metric for models trained across 5 random seeds for the same hyperparameter setting, as well as a 95\% confidence interval.  This figure compares training with a sparse reward vs shaping reward.}
    \label{fig:curves-quick-sort-reward}
\end{figure}

\begin{figure}[t]
    \centering
    \begin{tabular}{cc}
        \includegraphics[width=0.45\textwidth]{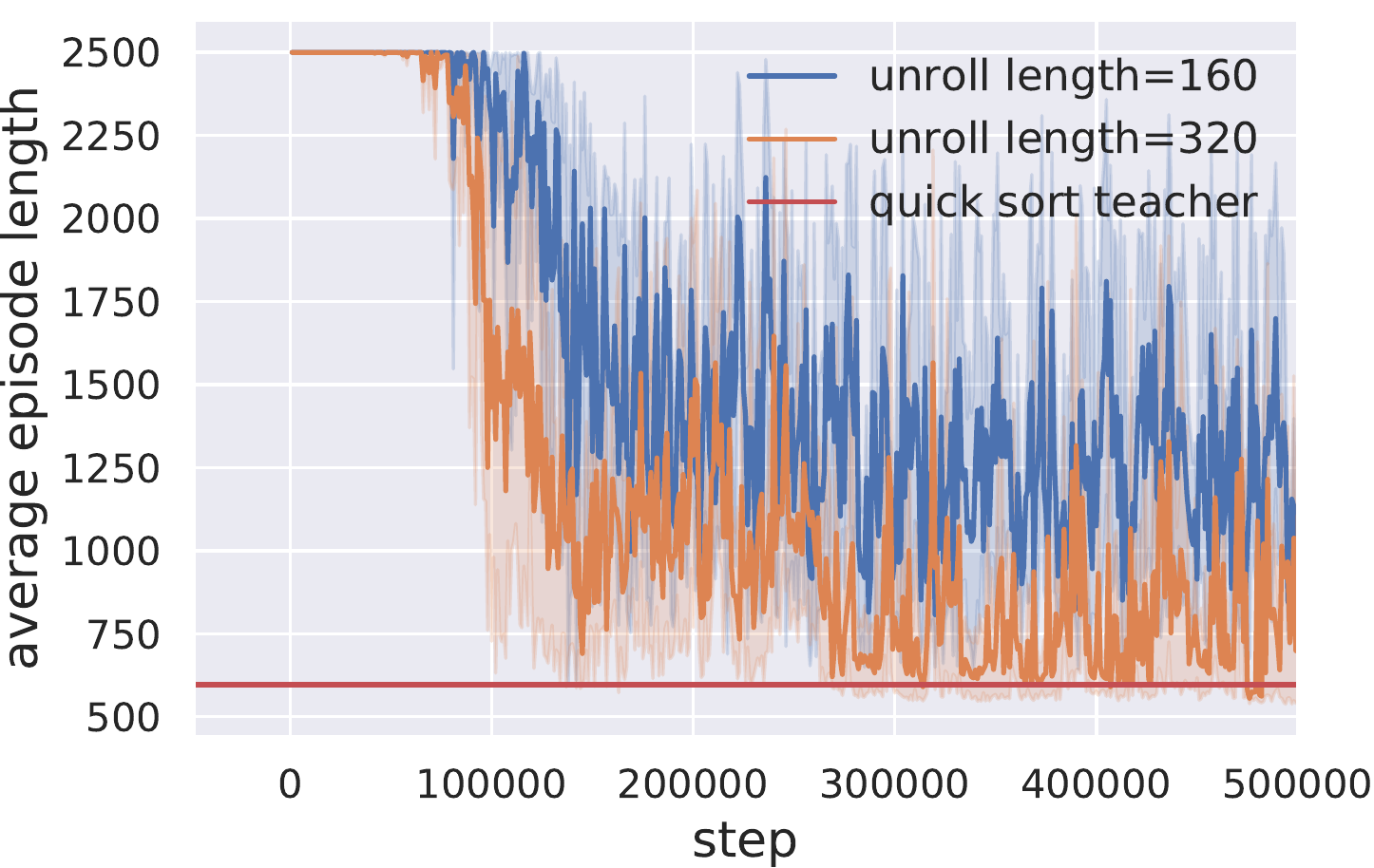} & 
        \includegraphics[width=0.45\textwidth]{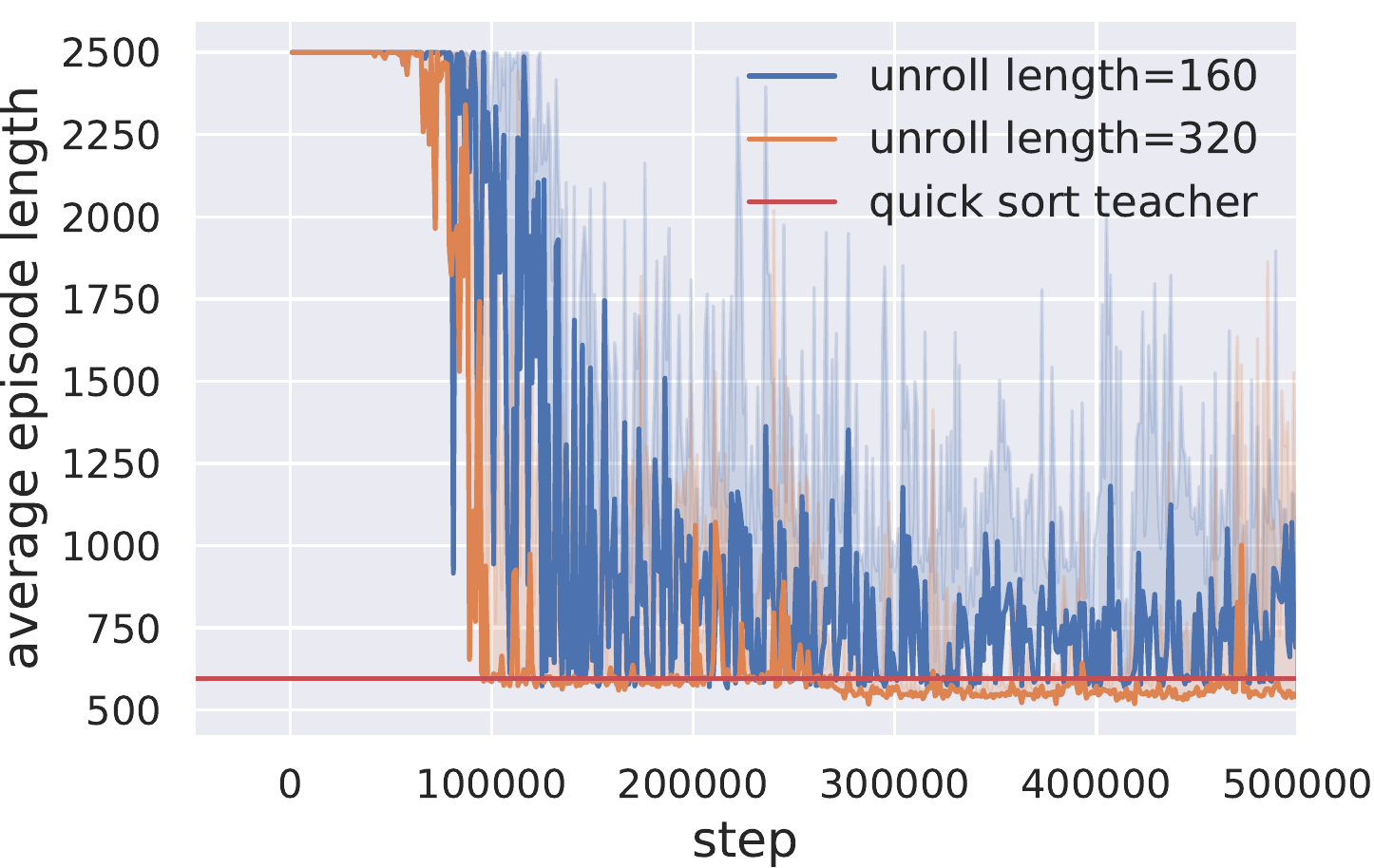}
        \\
        (a) Average across 5 seeds & (b) Best across 5 seeds
    \end{tabular}
    \caption{Typical training curves for the quick sort interface (\secref{sec:functions}). The curves show the mean(left) and minimum (right) of the average episode length metric for models trained across 5 random seeds for the same hyperparameter setting, as well as a 95\% confidence interval.  This figure compares training with an unroll length of 160 steps vs 320 steps.}
    \label{fig:curves-quick-sort-unroll-len}
\end{figure}

In \figref{fig:curves-quick-sort-reward} and \ref{fig:curves-quick-sort-unroll-len} we show
typical training curves for the quick sort interface (\secref{sec:functions}).  The models are trained on instances of size 30-50 and during training the maximum episode length is set to 2,500.  In this case we can see that with the shaping reward, the model learned to do sorting much faster, but with sparse reward the model is able to achieve significantly better performance, hence learning better algorithms.  Again if we look at the figures on the right hand side we can see that our learning process is able to find good neural programs that outperform the teacher reliably.

Despite the strong empirical results, we cannot yet prove that our learned model implements the correct algorithm, in the sense that running the learned model would guarantee sorting the array range correctly, regardless of the initial conditions.  On the contrary, \citep{cai2017making} claimed provably correct guarantees by showing that the learned model imitates the teacher policy perfectly, and assuming the teacher is correct, the learned model is also correct.  We could use the same strategy to prove our model trained with pure supervised learning to be correct, by enumerating all the possible inputs and show that the model would output the same actions as the teacher.  However, the agents learned with RL no longer imitates the teacher, therefore the same strategy no longer works.  How to prove the learned agents are guaranteed to be correct is a challenging direction we leave for future work.

To conclude this section, we also included two videos showing the execution traces for the sorting task, under the bubble / insertion and quick sort interfaces separately, comparing the learned models with scripted agents.  These execution traces clearly show that the agents have learned different behaviors than the teacher, and can hopefully help clarify what the model is actually doing.

\end{document}